\documentclass[11pt]{article} 
\usepackage[utf8]{inputenc}
\usepackage[margin=1in]{geometry}  

\usepackage{microtype}
\usepackage{graphicx}

\usepackage{booktabs}
\usepackage{natbib}
\usepackage{xcolor}
\usepackage{wrapfig}

\usepackage{hyperref}

\usepackage{enumitem}
\usepackage{amsmath}
\usepackage{amsfonts}
\usepackage{xspace}
\usepackage{multirow}
\usepackage{amsthm}
\usepackage{array, makecell}
\usepackage{subcaption}
\usepackage{xurl}
\usepackage{algorithm}
\usepackage{algorithmic}

\newtheorem{assumption}{Assumption}
\newtheorem{definition}{Definition}
\newtheorem{property}{Property}
\newtheorem{remark}{Remark}
\newtheorem{theorem}{Theorem}[section]

\newtheorem{lemma}{Lemma}
\allowdisplaybreaks

\DeclareMathOperator*{\argmin}{arg\,min}

\usepackage{mathtools}
\DeclarePairedDelimiter\round{\lfloor}{\rceil}
\DeclarePairedDelimiter\ceil{\lceil}{\rceil}
\DeclarePairedDelimiter\floor{\lfloor}{\rfloor}

\usepackage{stmaryrd}

\usepackage{csquotes}

\begin{document}

\title{Towards Federated Learning with on-device Training and Communication in 8-bit Floating Point}

\author{
\begin{tabular}{cccc}
Bokun Wang \thanks{Work done when the author was an intern at Arm.} & Axel Berg \thanks{Corresponding author: \texttt{axel.berg@arm.com}} \thanks{This work was partially supported by the Wallenberg AI, Autonomous Systems and Software Program (WASP), funded by the Knut and Alice Wallenberg Foundation.} & Durmus Alp Emre Acar & Chuteng Zhou \\
\small Texas A\&M University & \small Arm / Lund University & \small Arm & \small Arm \\
\small \texttt{bokunw.wang@gmail.com} & \small \texttt{axel.berg@arm.com} & \small \texttt{durmusalpemre.acar@arm.com} & \small \texttt{chu.zhou@arm.com}
\end{tabular}
}
\date{}
\maketitle
\vspace{-2em}  
\begin{abstract}
Recent work has shown that 8-bit floating point (FP8) can be used for efficiently training neural networks with reduced computational cost compared to training in FP32/FP16. In this work, we investigate the use of FP8 training in a federated learning context. This approach brings not only the usual benefits of FP8 which are desirable for on-device training at the edge, but also reduces client-server communication costs due to significant weight compression. We present a novel method for combining FP8 client training while maintaining a global FP32 server model and provide convergence analysis. Experiments with various machine learning models and datasets show that our method consistently yields communication reductions of at least 2.9x across a variety of tasks and models compared to an FP32 baseline to achieve the same trained model accuracy.
\end{abstract}

\interfootnotelinepenalty=10000
\section{Introduction}
A large amount of data is generated daily on personal smartphones and other devices at the edge. This data is very valuable for training machine learning models to provide services such as better voice recognition \cite{leroy2019federated} or text completion \cite{hard2018federated}. However, the local data often carries sensitive personal information which needs to be protected for privacy reasons. Furthermore, communication of local data between billions of devices and data centers is expected to occupy lots of network bandwidth and transmission is costly in terms of power consumption, which is a primary concern for edge devices running on batteries.  

Despite these constraints, it is still possible to train a model without having to transmit this local data using federated learning \cite{mcmahan2017communication}. In federated learning, each local device performs training locally with their local data and update their local models. When it comes to communication, the central server receives local models from a subset of devices. The central server then aggregates these local models and transmits a new global model back to those devices for a model update. In this way, no local data is ever exposed during communication and the global model can learn from local data as communication goes on. 

Since its inception, new techniques around federated learning have been proposed to reduce communication cost. The local models, albeit smaller than the local data, are still expensive to transmit via wireless communication and will be taxing on local devices’ battery life if performed very frequently. One method to reduce communication cost is to quantize the models before the communication occurs, and several works have shown that this can be done with limited reduction of model accuracy \cite{gupta2023quantization, honig2022dadaquant, ni2024fedaqt}. 

\begin{figure}[t]
    \centering
    \resizebox{.85\textwidth}{!}{   
    \includegraphics[width=\linewidth, trim={0, 0, 0.5cm, 0},clip]{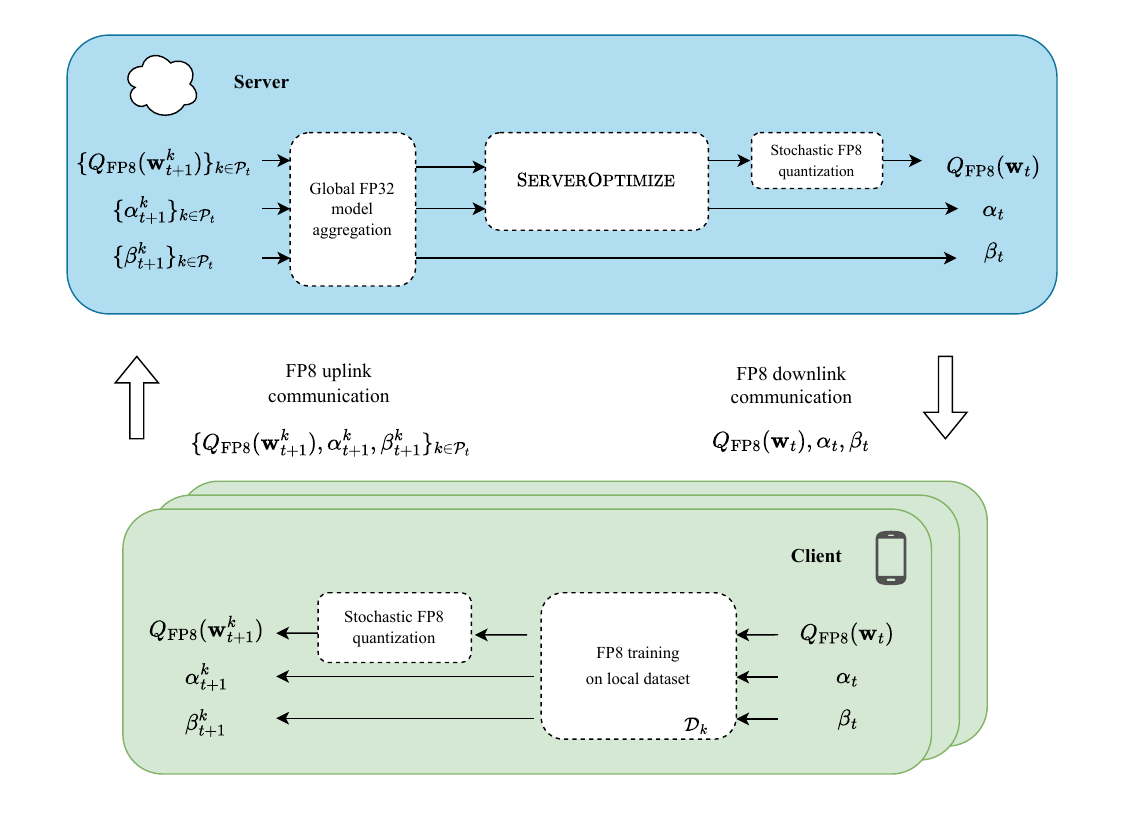}
    }
    \caption{Overview of the proposed federated learning with local FP8 on-device training and weight quantization in both uplink and downlink communication. For time step $t$, the server sends quantized FP8 global model weights $Q_{\text{FP8}}(\mathbf{w}_{t})$ and range parameters $\alpha_t$, $\beta_t$ for the weights and activations respectively. Each client $k$ that participates in the training round then performs FP8 local training and sends back updated weight and range parameters. At the server, new global weights are obtained using global aggregation and an additional optimization step.}
    \label{overview}

\end{figure}

In this paper, we focus on the use of a new type of short floating point number format which has not yet been explored for federated learning: 8-bit floating point (FP8). An 8-bit floating point number is only ¼ of the length of a 32-bit floating point number. Therefore, it has smaller representation power and lower precision than full 32-bit floating point number format, but it offers great savings in terms of model storage and memory access. Computation can be greatly accelerated with the FP8 format because of significantly less bit-wise operations required compared to FP32/FP16. Application of FP8 number format to deep learning model training and inference is in a nascent stage but is widely expected to have fast growth. 

Being a very efficient training datatype, FP8 is a good candidate for on-device training at the edge and the wide industrial support \cite{micikevicius2022fp8} behind it points to wide-spread real-world applications. A future scenario where edge devices can perform efficient on-device training with native hardware support introduces a new class of federated learning problems. It also increases device heterogeneity in a federated learning setting, where participating devices and servers may have different levels of hardware support for FP8.

In this work, we introduce an implementation of federated learning with on-device FP8 training, which learns from quantized models effectively while being efficient in its communication and computing cost. A high-level overview of this method is shown in Figure \ref{overview}. We summarize our main contributions as follows:

\begin{enumerate}
    \item A novel method for combining local FP8 client training with an FP32 server model. The local FP8 training is simulated by quantization-aware training (QAT) and all communication between the server and the clients is done in FP8. Furthermore, we provide an additional optimization procedure for weight aggregation on the server that alleviates potential performance drops caused by quantization, without affecting communication costs.
    \item We provide convergence analysis and motivation for the use of stochastic quantization for communication, but deterministic quantization for QAT. These design choices are further supported by experimental ablation studies.
    \item We proved experiments on image and audio classification benchmarks, showing minimal loss in performance while obtaining large reductions in communication costs for both convolutional and transformer-based models. 
\end{enumerate}

\section{Related Work}

\subsection{Federated Learning with Quantized Communication}
Quantizing model weights is an effective way of reducing the communication cost of federated learning. For example, quantizing weights from 32 to 8 bits immediately reduces the number of communicated bits by 75\% per training round. However, in practice the same functional performance is often not reached by the quantized model, because quantization can cause slower convergence of the training process. As a consequence, the communication reduction obtained to yield a particular performance threshold is often far from the ideal.

An important reason for the slower convergence obtained with quantized communication is that quantization of the model weights will introduce a bias term, which makes the server model a biased estimate of the average client model. In order to alleviate this problem, the use of stochastic rounding has been proposed \cite{zheng2020design, bouzinis2023wireless}. Hence, when aggregating the client models at the server, the stochastic rounding errors tend to zero as the number of clients grows large. This method has also been shown to improve the learning process from a privacy perspective \cite{youn2023randomized}. Similar work, \cite{he2020cossgd}, has shown that stochastic rounding in conjunction with non-linear quantization can reduce the number of communication rounds to reach convergence even more. Nevertheless, there is a limited amount of research on how to effectively combine quantized communication with low-precision client training.

\cite{yoon2022bitwidth} investigated a similar setup as ours, i.e.\ low-precision local training with quantized communication, but only for uniform quanization schemes, such as INT4 and INT8. This leads to sub-optimal performance compared to full-precision training, since many neural networks are known to be sensitive to integer quantization, especially gradients during training are thought to require a bigger dynamic range \cite{kuzmin2022fp8}. For this reason, the industry is coalescing around FP8 for efficient training hardware \cite{micikevicius2022fp8}. This motivates an investigation of local training in low-precision floating point, which provides better training dynamics at the same communication cost.

\subsection{FP8 Quantization for Neural Networks}

While integer representations, have been widely adopted for neural network quantization for efficient inference, the use of FP8 remains relatively new in comparison and has been more focused on model training. A few recent works \cite{wang2018training, sun2019hybrid, peng2023fp8} have proposed centralized neural network training in FP8 with promising results. However, for some networks, a single FP8 representation was found to be insufficient for retaining accuracy in certain operations. Notably, the backwards pass through the network typically requires higher dynamic range than the forward pass. To this end, a binary interchange FP8 format that uses both E4M3 (4-bit exponent and 3-bit mantissa) and E5M2 (5-bit exponent and 2-bit mantissa) representations has been proposed, which allows for minimal accuracy drops compared to FP16 across a wide range of network architectures \cite{micikevicius2022fp8}. Concurrent work, \cite{kuzmin2022fp8}, proposed a similar solution, where the exponent bias is flexible and updated for each tensor during training, which allows for maintaining different dynamic ranges in different parts of the network.

An important factor in neural network quantization that is often overlooked, is that not all model architectures are equally sensitive to quantization, since the distribution of weights and activations are often architecture-dependent. As a consequence, a quantization scheme that works well for e.g.\ convolutional networks might not be suitable for attention-based models. \cite{shen2024efficient} provided a thorough investigation of different quantization schemes for models on a wide variety of applications and found that FP8 quantization, with a combination of E4M3 and E3M4 in particular, was least detrimental to model performance. Similar results were found by \cite{nikolic2024schrodinger} by dynamically adjusting the number of exponent and mantissa bits during training, which lead to large reductions in both memory footprint and energy consumption. This makes FP8 training particularly interesting in a federated training setup, where the available hardware resources are often limited.

\section{Method}
\subsection{Preliminaries}
Consider the federated learning problem, where $K$ clients update their local models by training on disjoint local datasets $\{\mathcal{D}_k\}_{k=1}^K$. Each client minimizes their own local objective functions $F_k(\mathbf{w}, Q, \alpha, \beta) = \mathbb{E}_{(\mathbf{x}, y) \sim \mathcal{D}_k}[l(\mathbf{w}; \mathbf{x}, y, Q, \alpha, \beta)]$, where $\mathbf{w}$ are the model parameters, $Q$ is a quantization operator and $l$ is the loss function. Furthermore, $\alpha$ and $\beta$ are the per-tensor clipping values used for maintaining the dynamic range when quantizing weights and activations respectively. Henceforth, we denote quantized weights based on range $\alpha$ as $Q(\mathbf{w}; \alpha)$. In practice, each layer of the network has its own unique clipping parameters for both the weights and activations, but we omit this in our notation for readability. 

We consider a modification of the standard federated averaging (FedAvg) \cite{mcmahan2017communication} with quantized weights, where the objective is to find

\begin{equation}
\begin{aligned}
    \min_{\mathbf{w}} \quad &F(\mathbf{w}, Q, \alpha, \beta), \\
    F(\mathbf{w}, Q, \alpha, \beta) &= \sum_{k=1}^K \frac{n_k}{n} F_k(\mathbf{w}, Q, \alpha, \beta),
\end{aligned}
\label{standard}
\end{equation}
where $n_k = |\mathcal{D}_k|$ is the number of training samples on the $k$:th device and $n = \sum_k n_k$ is the total number of training examples.

\subsection{Floating Point Representation}
A floating point number $x$ with $e$ exponent bits and $m$ mantissa bits can be written as
\begin{equation}
x = (-1)^s 2^{p-b} \left( 1 + \frac{d_1}{2} + \frac{d_2}{2^2} + \hdots + \frac{d_m}{2^m} \right),
\end{equation}
where $s \in \{0, 1\}$ is the sign bit, $d_i \in \{0, 1\}$ is the mantissa, $p \in \{0, 1, \hdots, 2^{e-1}\}$ is the exponent and $b$ is the exponent bias. In addition, we assume that $p=0$ is reserved for subnormal numbers, which allows an exact representation of 0 and other special values.

Compared to integer quantization, the quantization grid for floating point numbers is not uniform. Increasing the number of exponent bits results in a higher dynamic range, whereas increasing the number of mantissa bits increases the precision. Therefore, give a fixed budget on the number of bits to allocate, there is a trade-off to be made between the two.  

\subsection{On-Device Quantization-Aware Training}
Depending on the hardware support, on-device local training can be performed in native FP8 or using quantization-aware training (QAT), or a mix of the two. Native FP8 training is supported by the latest AI hardware in data centers such as Nvidia's H100/H200 series of GPUs. There is significant industry support behind FP8 and it is only a matter of time for FP8 hardware support to arrive on edge devices. 

For research purposes, we here resort to QAT, and follow the FP8 QAT method described in \cite{kuzmin2022fp8}, using per-tensor quantization for both model weights and activations, with one signed bit, and $m=3$ and $e=4$ bits for the mantissa and exponent respectively, as well as a flexible exponent bias that depends on learnable clipping parameters. QAT with both weights and activations quantization is a good way of simulating native FP8 training on supported hardware with low precision arithmetic. In our setting, we are not simulating the effect of gradient quantization which happens on FP8-enabled hardware. However, previous work, \cite{kuzmin2022fp8}, has shown that it is a good approximation to ignore its effect. 

Let $\mathbf{x} = (x_1, ..., x_d)^T \in \mathbb{R}^d$ denote an FP32 input tensor and $Q_{\text{det}}: \mathbb{R}^d \rightarrow \mathbb{R}^d$ the FP8 deterministic quantization operator with a clipping parameter $\alpha$, whose element-wise outputs are given by 
\begin{align}
Q_{\text{det}}(x_i; \alpha) = 
\begin{cases}
-\alpha, \quad  x_i \leq -\alpha, \\
s_i \round*{\frac{x_i}{s_i}}, \\  
\alpha, \quad x_i \geq \alpha,\\
\end{cases}
\end{align}
where $\round{\cdot}$ denotes rounding to the nearest integer. Here, $s_i = 2^{p_i}$ is a scale factor that is applied before and after rounding. The scale factor can be computed from $p_i$ using
\begin{align}\label{eq:scale}
    p_i = \begin{cases}
        \floor*{\log_2 |x_i|+b} - b - m, &\floor*{\log_2 |x_i| {+} b} > 1 \\
        1 - b - m, &\text{otherwise,}
    \end{cases}
\end{align}

\begin{figure*}
\begin{minipage}[t]{\textwidth}
\begin{algorithm}[H]
\caption{\textsc{LocalUpdate}}\label{alg-device}
\begin{algorithmic}[1]
\REQUIRE $\mathbf{w}, \alpha, \beta, Q, \mathcal{D}$, number of minibatches $P$
\STATE Set  $\mathbf{w}_1 = \mathbf{w}, \alpha_1 = \alpha, \beta_1 = \beta$
\FOR{$p = 1, ..., P$}
    \STATE Sample minibatch $\mathcal{B}_p$ randomly from $\mathcal{D}$
    \STATE Do forward pass with $Q$ on the minibatch $\mathcal{B}_p$
    \STATE Do backwards pass using STE and update $\mathbf{w}_{p+1}, \alpha_{p+1}, \beta_{p+1}$
\ENDFOR
\STATE Return $\mathbf{w}_{P+1}, \alpha_{P+1}, \beta_{P+1}$
\end{algorithmic}
\end{algorithm}
\end{minipage}
\end{figure*}
where the exponent bias depends on the clipping value $\alpha$ as $b = 2^e - \log_2 \alpha + \log_2 (2 - 2^{-m}) - 1$. At training time, $\alpha$ is first initialized using the maximum absolute value of each weight range, and then treated as a learnable parameter that is updated during training. Furthermore, the gradients of the non-differentiable rounding operators are computed using the straight-through estimator (STE) \cite{bengio2013estimating}, i.e.\ $\frac{\partial \round*{x_i}}{\partial x_i} = 1$, with a key exception being $\log_2 |x_i|$, which is treated as a constant following a similar approach as in \cite{kuzmin2022fp8}. Activations are quantized using the same procedure, but with a separate clipping value $\beta$. A summary of the local on-device training procedure is given in Algorithm \ref{alg-device}.

\begin{figure*}
\begin{algorithm}[H]
\caption{\textsc{ServerOptimize}}\label{alg-server}
\begin{algorithmic}[1]
\REQUIRE $\{\mathbf{w}_t^k, \alpha_t^k, n_k\}_{k \in \mathcal{P}_t}, Q$
\STATE Compute $m_t = \sum\limits_{k} n_k$
\STATE Using gradient descent, update the weights as 
$ \mathbf{w}_{t+1} \leftarrow \argmin\limits_{\mathbf{w}} \sum\limits_{k} \frac{n_k}{m_t} \| Q(\mathbf{w}; \alpha_{t+1}) - Q(\mathbf{w}^k_t; \alpha^k_t) \|^2_2 $ 
\STATE Set $S \leftarrow [\min\limits_{k} \alpha_t^k, \max\limits_{k} \alpha_t^k]$
\STATE Using grid search, update the parameters as 
$\alpha_{t+1} \shortleftarrow \argmin\limits_{\alpha \in S} \sum\limits_{k \in \mathcal{P}_t} \frac{n_k}{m_t}\| Q(\mathbf{w}_{t+1}; \alpha) - Q(\mathbf{w}^k_t; \alpha^k_t) \|^2_2 $
\STATE Return $\mathbf{w}_{t+1}, \alpha_{t+1}$
\end{algorithmic}
\end{algorithm}
\end{figure*}

\subsection{Unbiased Quantized Communication}
When applying FP8 QAT to a federated leaning scenario, an important aspect is the ability to reduce communication overhead by transferring weights between clients and the server using only 8 bits per scalar value. 
On client devices with hardware for FP8 mixed-precision training support, a copy of high-precision master weights \cite{NVMixed} is present as in our QAT setup. 
Therefore, at the end of each communication round, the participating clients need to perform a hard reset of their master weights to the de-quantized FP8 values on a quantization grid.
This approach allows for cost reduction in both the uplink and downlink communication.

At each communication round $t$, 
active clients $\mathcal{P}_t$ will send the FP8-quantized 
weights to the central server together with the clipping parameters. However, to form an unbiased estimate of the average client weight, we need a different quantization operator. We therefore introduce stochastic quantization as
\begin{align}
    Q_{\text{rand}}(x_i; \alpha) = s_i 
    \begin{cases} 
      \ceil*{\frac{x_i}{s_i}} & \kappa \leq \frac{x_i}{s_i} - \floor*{\frac{x_i}{s_i}} \vspace{0.2em} \\
      \floor*{\frac{x_i}{s_i}} & \text{otherwise,}
   \end{cases}
\end{align} for $-\alpha \leq x_i \leq \alpha$ and where $\kappa$ is a Bernoulli random variable that takes the values 0 and 1 with equal probability. It is straightforward to verify that this randomized quantization is unbiased from a statistics point of view, i.e.\ for $-\alpha \leq x_i \leq \alpha$ we have $\mathbb{E}[Q_{\text{rand}}(x_i; \alpha)] = x_i$, while the deterministic quantization introduced earlier is biased\footnote{
 The unbiasedness of stochastic quantization assumes a finite domain of the input variable, such that no clipping occurs. While clipping may occur for some values in practice, this does not affect the majority of the weights or activations. In order to simplify theoretical analysis, the assumption of no clipping is often made in the literature, see for example \cite{li2017training}.}.

The quantized weights are then aggregated at the server using a weighted federated average as
\begin{align}
\mathbf{w}_{t+1} \leftarrow \sum_{k \in \mathcal{P}_t} \frac{n_k}{m_t}  Q_{\text{rand}}(\mathbf{w}^k_{t+1}; \alpha^k_{t+1}),
\end{align}
where $m_t=\sum_{{k'} \in \mathcal{P}_t} n_{k'}$. The clipping values are also aggregated, but without quantization, since their contribution to the communication overhead is small relative to the weights. The aggregated weights are then quantized again to FP8 and transmitted to the next set of active clients with a new set of quantization parameters.

An illustrative example of weight quantization is shown in Figure \ref{weight-figure}a-d, where different weight distributions and the corresponding quantization errors are shown for both deterministic and stochastic quantization. For a given scalar weight, deterministic quantization always results in a lower quantization error, but for communication purposes we are more interested in the quantization error of the aggregated server weights. As can be seen in Figure \ref{weight-figure}d, stochastic quantization yields a lower error for the average model. For illustration purposes, we have here assumed that the weights are identical on each device, which is not the case in practice. Nevertheless, with stochastic quantization, the average quantized model will be an unbiased estimate of the average unquantized model.

\begin{figure*}[t]
    \centering
    \begin{subfigure}[t]{0.24\linewidth}
    \centering
    \includegraphics[width=\linewidth]{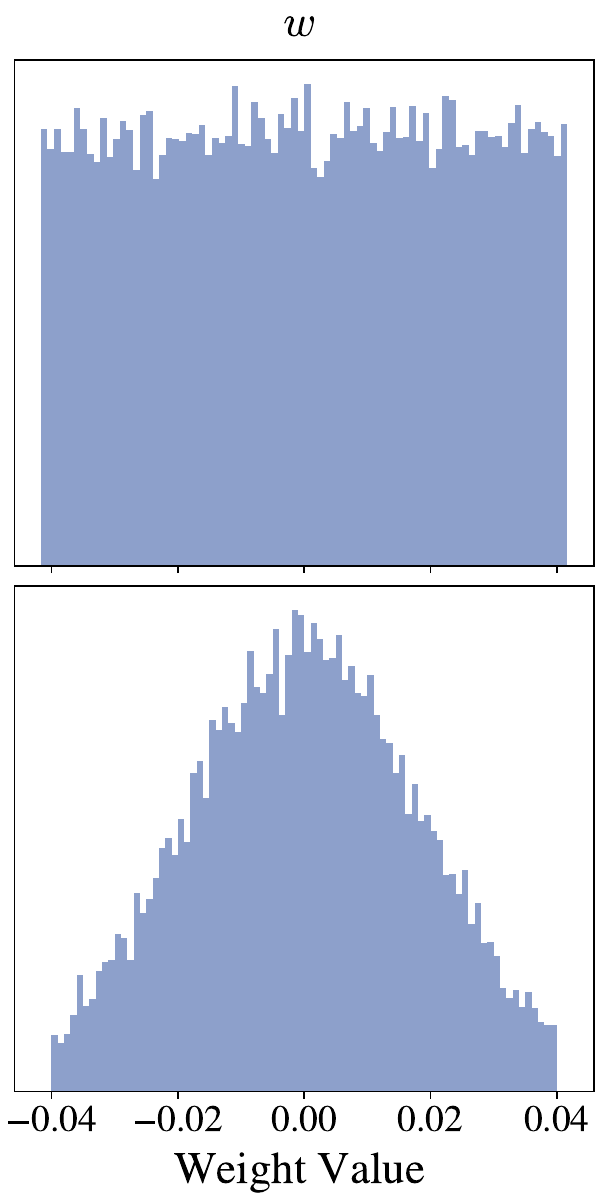}
    \caption{FP32 weight distribution}
    \end{subfigure}
    \begin{subfigure}[t]{0.24\linewidth}
    \centering
    \includegraphics[width=\linewidth]{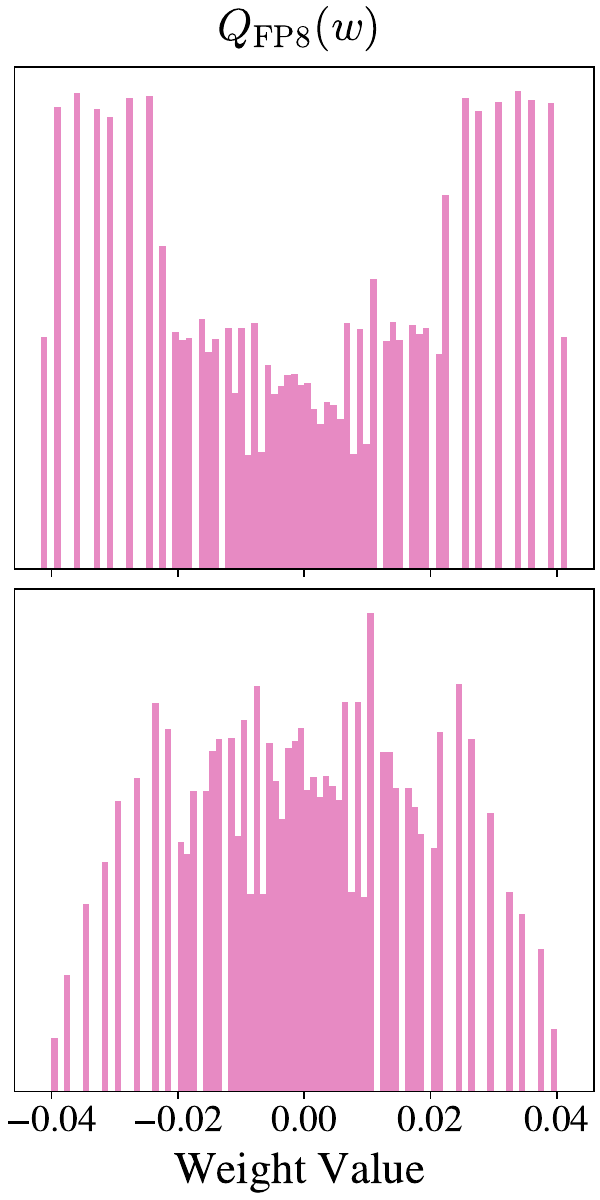}
    \caption{FP8 weight distribution}
    \end{subfigure}
    \begin{subfigure}[t]{0.24\linewidth}
    \centering
    \includegraphics[width=\linewidth]{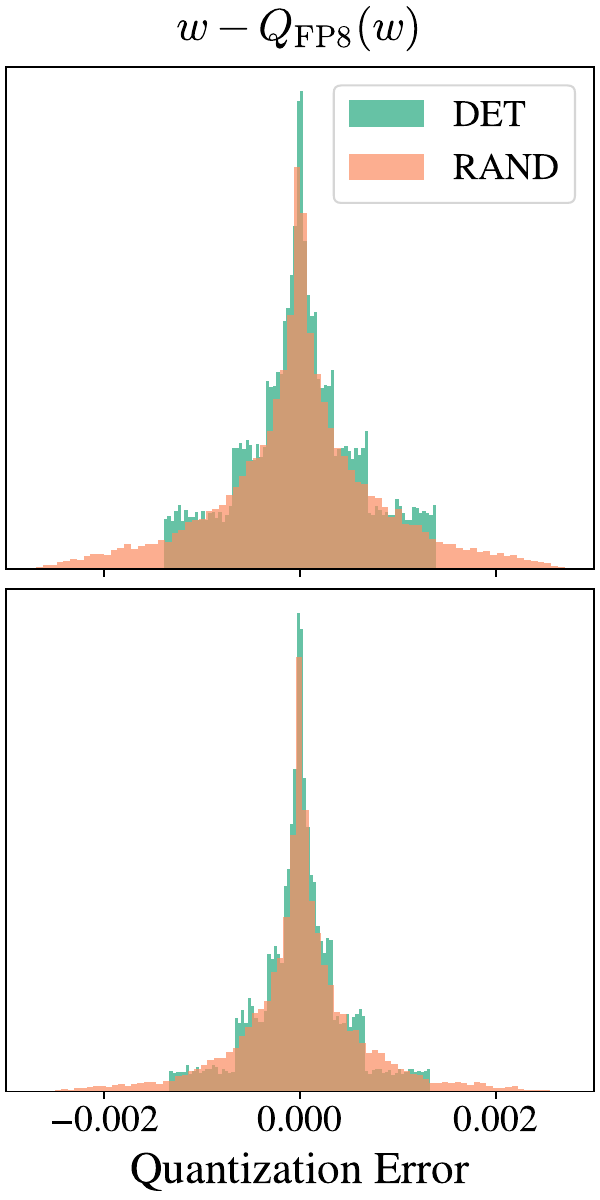}
    \caption{Quantization error \\ distribution.}
    \end{subfigure}
    \begin{subfigure}[t]{0.24\linewidth}
    \centering
    \includegraphics[width=\linewidth]{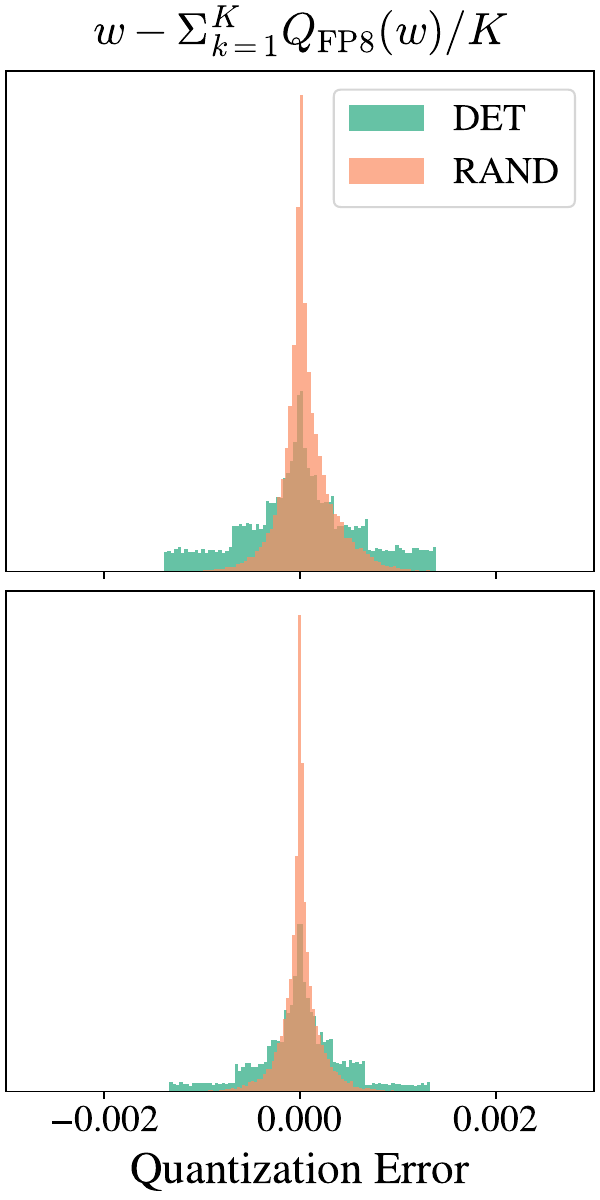}
    \caption{Quantization error \\ distribution of average model.}
    \end{subfigure}
    \caption{Examples of FP8 quantized weights using the E4M3 format and the corresponding quantization errors for different weight distributions. The range $\alpha$ of the quantization is calculated such that the entire distribution fits within the dynamic range of the FP8 representation. In general, deterministic quantization yields the lowest quantization error for any distribution. However, when weights are quantized independently on several devices (here we use $K=10$) and averaged, stochastic quantization yields smaller errors. Top row: uniform weight distribution, which is a common initialization for convolutional layers. Bottom row: truncated normal distribution, which is commonly used to initialize the linear layers in Transformer models.}
    \label{weight-figure}
\end{figure*}

\subsection{Server-Side Optimization (\textsc{ServerOptimize})}
The standard federated average of the weights in the un-quantized scenario corresponds to the minimization of weighted mean squared error (MSE) between the server parameters and the client parameters. However, when the server parameters are quantized before transmission to the clients, this property no longer holds. We therefore propose a modification to the server-side model aggregation, where the MSE is explicitly minimized. This can be done without communication overhead since all computations are done on the server. Another advantage of this approach is that it leverages the computing power of the server to do more optimization, since the server typically possesses more computing power than a client device. 

At time step $t$, we perform model/parameters aggregation to obtain $\mathbf{w}_{t+1}, \alpha_{t+1}$ for the next communication round by minimizing the following mean-squared error (MSE).

\begin{align*}
\mathbf{w}_{t+1}, \alpha_{t+1} =
\argmin_{\mathbf{w},\alpha} \sum_{k \in \mathcal{P}_t} \frac{n_k}{m_t} \| Q_{\text{rand}}(\mathbf{w}; \alpha) - Q_{\text{rand}}(\mathbf{w}^k_t; \alpha^k_t) \|^2_2.
\end{align*}
Note that when there is no communication quantization, the closed-form solution to \textsc{ServerOptimize} is the federated average $\mathbf{w}_{t+1} = \sum_{k\in\mathcal{P}_t} \frac{n_k}{m_t}\mathbf{w}_t^k$. Since the problem above has no closed-form solution for the quantized communication case, we use the \emph{alternating minimization} approach to optimize $\mathbf{w}$ and $\alpha$.
First, we optimize the model weights $\mathbf{w}$, while fixing $\alpha$ to $\alpha_{t+1} = \sum_{k\in\mathcal{P}_t} \frac{n_k}{m_t} \alpha_{t+1}^k$. This is done by minimizing the MSE as 
\begin{align}
\mathbf{w}_{t+1} = 
\argmin_{\mathbf{w}} \sum_{k \in \mathcal{P}_t} \frac{n_k}{m_t} \| Q_{\text{rand}}(\mathbf{w}; \alpha_{t+1}) - Q_{\text{rand}}(\mathbf{w}^k_t; \alpha^k_t) \|^2_2.
    \label{w_optimize}
\end{align}
This minimization can be done using gradient descent with a fixed number of iterations
 
Next, we aim to optimize $\alpha$ while fixing $\mathbf{w}$ to $\mathbf{w}_{t+1}$. 
However, minimizing the MSE with respect to $\alpha$ using gradient descent would require access to $\frac{\partial s_i}{\partial \alpha}$, which is non-differentiable at multiple points and therefore highly numerically unstable. We instead perform a grid search over a fixed set of clipping values uniformly distributed in $S = [\min_{k \in \mathcal{P}_t} \alpha_t^k, \max_{k \in \mathcal{P}_t} \alpha_t^k]$ as

\begin{align}
\alpha_{t+1} = \argmin_{\alpha \in S} \sum_{k \in \mathcal{P}_t} \frac{n_k}{m_t}\| Q_{\text{rand}}(\mathbf{w}_{t+1}; \alpha) - Q_{\text{rand}}(\mathbf{w}^k_t; \alpha^k_t) \|^2_2.\label{alpha_optimize}
\end{align}

The \textsc{ServerOptimize} routine is given in Algorithm \ref{alg-server}, which takes place completely on the server and can be used as an optional step in order to improve aggregation of the model weights.

\subsection{Overall algorithm}
A summary of our proposed FP8 FedAvg with unbiased communication (FP8FedAvg-UQ) method is presented in Algorithm \ref{alg}, where the optional server-optimization step (UQ+) corresponds to replacing the quantized federated averaging of the weight and range parameters with our two-step MSE minimization optimization in Equations \eqref{w_optimize} and \eqref{alpha_optimize}. It is important to note that our method involves two different quantization operators. On-device QAT uses a deterministic and biased quantizer, while the communication part adopts its stochastic counterpart which is unbiased. In the next section, we will give a convergence analysis result for FP8FedAvg-UQ and show that these design choices are well-motivated from a theoretical perspective.

\begin{algorithm}[t]
\caption{{\color{red}FP8FedAvg-UQ}, {\color{blue} FP8FedAvg-UQ+}}\label{alg}
\begin{algorithmic}[1]
\REQUIRE $\mathbf{w}_1, \alpha_1, \beta_1, Q_{\text{det}}, Q_{\text{rand}}$
\FOR{$t = 1, ..., T$}
    \STATE Sample a set $\mathcal{P}_t \in [K]$ of $P$ active devices
    \FOR{each client $k \in \mathcal{P}_t$}
        \STATE Receive $Q_{\text{rand}}(\mathbf{w}_t; \alpha_t), \alpha_t, \beta_t$ from server
        \STATE $\{ \mathbf{w}^k_{t+1}, \alpha^k_{t+1}, \beta^k_{t+1} \}  \leftarrow \textsc{LocalUpdate}(\mathbf{w}_t, Q_{\text{det}}; \alpha_t, \beta_t, \mathcal{D}_k)$
        \STATE Send $Q_{\text{rand}}(\mathbf{w}^k_{t+1}; \alpha^k_{t+1}), \alpha^k_{t+1}, \beta^k_{t+1}$ to server
    \ENDFOR
    \STATE Compute $m_t = \sum\limits_{k\in\mathcal{P}_t} n_k$
    \STATE Compute $\beta_{t+1} \leftarrow  \sum\limits_{k \in \mathcal{P}_t} \frac{n_k}{m_t}\beta^k_{t+1}$
    \STATE {\color{red} Compute $\mathbf{w}_{t+1} \leftarrow \sum\limits_{k \in \mathcal{P}_t} \frac{n_k}{m_t}  Q_{\text{rand}}(\mathbf{w}^k_{t+1}; \alpha^k_{t+1})$ }
    \STATE{\color{red} Compute \indent \indent $\alpha_{t+1} \leftarrow  \sum\limits_{k \in \mathcal{P}_t} \frac{n_k}{m_t} \alpha^{k}_{t+1}$}
    \STATE {\color{blue}$\{ \mathbf{w}_{t+1}, \alpha_{t+1} \}  \leftarrow \textsc{ServerOptimize}(\{\mathbf{w}^k_{t+1}, \alpha^k_{t+1}, n_k\}_{k \in \mathcal{P}_t}, Q_{\text{rand}})$}
\ENDFOR
\STATE Evaluate on $\mathbf{w}_{T+1}, \alpha_{T+1}, \beta_{T+1}$
\end{algorithmic}
\end{algorithm}

\section{Convergence analysis and theoretical motivations}
We briefly state our main convergence theorem here and refer to the Appendix for formal assumption definitions and proof. Please note that we make the simplifying assumption to only consider weight quantization in our proof, which is standard for this type of theoretical analysis.

\begin{theorem}[Convergence of FP8FedAvg-UQ]\label{thm.main}
    \textbf{
    For convex and $L$-smooth federated losses in  \eqref{standard}
    with $G$-bounded unbiased stochastic gradients using an FP8 deterministic quantization method during training and an FP8 unbiased quantization method with bounded scales for model communication, the objective gap $\mathbb{E}\left[F(Q_{\text{rand}}(\mathbf{w}_\tau)) - F(\mathbf{w}_*)\right]$ decreases at a rate of 
    \begin{align*}
        O\biggl(\underbrace{\frac{\Delta^2+G^2U}{\sqrt{TU}} +\frac{UG^2L}{T}}_{\mathcal{T}_1}+\underbrace{\frac{GU^{2.5}S\sqrt{d}L}{\sqrt{T}}}_{\mathcal{T}_2}+\underbrace{S\sqrt{d}G}_{\mathcal{T}_3}\biggl),
    \end{align*}
    where $
    \Delta=\|\mathbf{w}_{1}-\mathbf{w}_{*}\|_2$ is the initial and optimal model difference, $\tau$ is uniformly sampled from $\{1,2,\dotsc,T\}$, $T$ is the number of rounds, 
    $U$ is the total number of updates done in each round, the quantization scales $s_i$ are uniformly bounded by $S$, $\mathbf{w}_1$ is the initial model, and $\mathbf{w}_*$ is an optimal solution of \eqref{standard}.
    }
\end{theorem}

\textbf{Proof Structure.} The proof builds on upper bounding a drift quantity   similar to the one defined in \cite{karimireddy2020scaffold} as\footnote{See Appendix \ref{single-round} on ${\mathbf{w}}_{t,u}^k$ definition for inactive devices.},
\begin{align}
    V_t=\frac{1}{KU}\sum_{k\in[K]}\sum_{u\in[U]}E\left\|Q_{\text{rand}}\left({\mathbf{w}}_{t}\right)-Q_{\text{det}}\left({\mathbf{w}}_{t,u}^k\right)\right\|^2_2\nonumber.
\end{align}
Note that if local models diverge, we get a higher $V_t$. If there is no quantization and local models converge, we get $V_t=0$. We focus on $V_t$ and the server model $\mathbf{w}$ in a single communication round and give the following Lemma as,

\begin{lemma}
\textit{For a setting that satisfies the assumptions listed in Theorem \ref{thm.main}, we have,
\begin{flalign}
  V_t \le 18&U^3S\sqrt{d} G \eta_t  + 9U^2\eta_t^2G^2, &&\label{m.A.2}\\
  E\|\mathbf{w}_{t+1}-\mathbf{w}_{*}\|^2_2- E\left\|{\mathbf{w}}_{t}-\mathbf{w}_{*} \right\|^2_2\le&\eta_tL U V_t\label{m.A.1}+2S\sqrt{d} G U \eta_t+\eta_t^2U^2G^2&&\nonumber\\
  &-2U\eta_tE\left(F\left(Q_{\text{rand}}\left({\mathbf{w}}_{t}\right)\right)-F\left(\mathbf{w}_{*}\right)\right), &&
\end{flalign}
}
\end{lemma}

where $\eta_t$ is the learning rate. By combining the equations with proper coefficients, we get a telescoping sum on the difference between the server model and the optimal model, as well as the accumulation of the loss with respect to the optimal model. This leads to Theorem \ref{thm.main} and we provide the full proof in Appendix \ref{A.flq}.

\begin{remark}\label{r.t.1}
$\mathcal{T}_1$ is a term similar to SGD convergence where it decreases with $O\left(\frac{1}{\sqrt{T}}\right)$ and depends on the bound on the second moment of stochastic gradient $G$, the smoothness  $L$, as well as the $\ell_2$-distance between the initial model $w_1$ and the optimal solution $w_*$.
\end{remark}
\begin{remark}
    $\mathcal{T}_2$ and $\mathcal{T}_3$ are terms that arise due to quantization. Due to \eqref{eq:scale} and the definition of $S$, the terms $\mathcal{T}_2$ and $\mathcal{T}_3$ exponentially decay when the number of mantissa bits $m$ increases, i.e., $\mathcal{T}_2\propto 2^{-m},\mathcal{T}_3\propto 2^{-m}$. 
\end{remark}
\begin{remark}\label{r.unbiased} We emphasize that unbiased quantization during communication is crucial. In the case of biased communication, the convergence proof does not hold and one can construct even diverging cases~\cite{beznosikov2023biased} for FedAvg. To ensure convergence for biased communication, we need more sophisticated techniques such as error feedback~\cite{richtarik2021ef21}.
\end{remark}
\begin{remark}\label{r.det} 
    Deterministic quantization is used during training. We bound the norm of QAT quantization error in the proof. Since deterministic quantization has a smaller error norm than stochastic one, this motivates us to use deterministic quantization during training.
\end{remark}
As we shall see in the next section, we observe strictly worse results if we use stochastic quantization during training or deterministic quantization during model transmission in our experiments, which aligns with the remarks above.

\section{Experiments and Ablation Studies}

We test our proposed method on two different tasks: image classification and keyword spotting. For each task, we use two different models and perform experiment using both an i.i.d.\ and a non-i.i.d.\ dataset split across clients. In the i.i.d.\ setup, the dataset is split randomly across the set of clients. In the non-i.i.d.\ setup, we simulate a more realistic heterogenity across clients which is specific to each dataset.

\subsection{Datasets and models}

\textbf{Image classification.} For image classification, we use the CIFAR10 and CIFAR100 datasets \cite{Krizhevsky09learningmultiple}, which consist of 60 000 examples of 32x32 color images divided into 10 and 100 different classes respectively. For this task, we adopt two different convolutional networks: 1) LeNet \cite{lecun1998gradient}, which has 800K parameters, and ResNet18 \cite{he2016deep}, with 11M parameters. For the ResNet model, we replace batch normalization after the convolutional layers with GroupNorm \cite{wu2018group}, since this is known to work better in a federated setting with skewed data distributions \cite{hsieh2020non}.

\begin{wrapfigure}{l}{0.5\linewidth}
    \centering
    \includegraphics[trim={1cm 1cm 1cm 1cm},clip,width=0.8\linewidth]
    {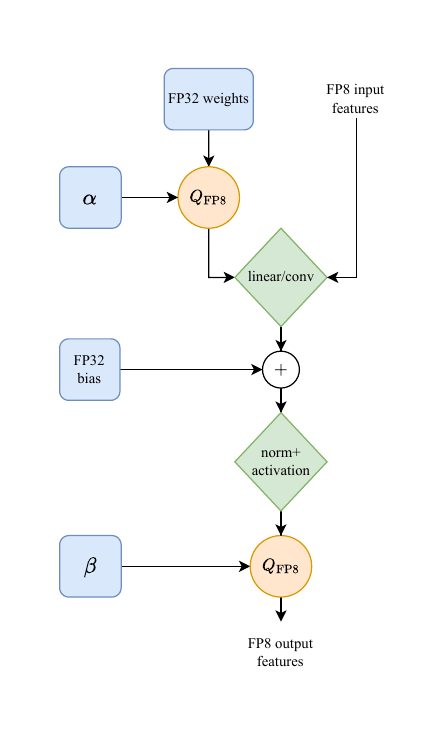}
    \caption{Illustration of the QAT method for a linear/convolutional layer. FP32 master weights are quantized using the range parameters $\alpha$, before being applied to the already quantized input. We then apply addition of bias, an optional normalization layer and the activation function in FP32, before quantizing to FP8 again using the range parameters $\beta$. In a real application, these operations could be performed using mixed-precision hardware.}
    \label{qat-diagram}

\end{wrapfigure}

In the (independent and identical and distributed) i.i.d. scenario, we use $K=100$ clients, a participation rate of $C=0.1$ in each round and train for $R = 1000$ rounds with a local batch size of $B=50$, where each client trains for 5 local epochs. In the non-i.i.d.\ image classification setting we sample the local datasets from a Dirichlet distribution with a concentration parameter of 0.3. Furthermore, we use SGD as the local optimizer with a constant learning rate of 0.1, weight decay of 0.001 and random cropping and horizontal flipping for data augmentation.

\textbf{Keyword spotting.} In order to apply our method to a more realistic federated learning task and more advanced model architectures, we resort to keyword spotting, where the task is to classify short snippets of audio as words in a small dictionary. For this, we use the Google SpeechCommands v2 dataset, which consists of 105 000 one-second recordings belonging to one of 35 classes \cite{warden2018speech}. Examples of classes in the dictionary include short words like \enquote{go}, \enquote{yes} and \enquote{on}. 

For the keyword spotting task, we train two different models: MatchboxNet3x1x64 \cite{majumdar20_interspeech}, which uses 1D time-channel separable convolutions with skip-connections, as well as the Keyword Transformer (KWT-1) \cite{berg21_interspeech} model, which consists of sequential transformer layers with time-domain self-attention. These two models have 350K and 450K parameters respectively and they both use mel-frequency cepstral coefficients (MFCCs) as input features. 

A known problem when training transformers is that SGD optimization requires many iterations and may fail to converge to a good solution \cite{kunstner2023noise}. Due to the relatively small number of local updates in a training round, SGD is therefore not a suitable optimization method for transformers in this setup. To achieve better training convergence, we instead use the momentum-based AdamW \cite{loshchilov2018decoupled} as local optimizer for the clients, where the states of the local optimizer for each active client are reset at the start of each communication round. Furthermore, we use an initial learning rate of 0.001 with a cosine decay scheduler and a weight decay of 0.1. For data augmentation, we apply SpecAugment \cite{park19e_interspeech} to the MFCCs during training.

When training in the i.i.d.\ scenario, we use the same setup for keyword spotting as for image classification. However, since the SpeechCommands datasets provides speaker identity as well, we can simulate a more realistic heterogenity for the non-i.i.d.\ dataset split. We therefore use the setup proposed in \cite{li22g_interspeech}, such that all recordings from a given speaker are assigned to a single client. This results in a total of $K=2112$ clients. In order to get a similar number of total training steps as in the i.i.d\ scenario, we reduce the participation rate to $C = 0.01$, the number of rounds to $R=500$ and the number of local gradient updates to 50.

\subsection{Mixed Precision Quantization Implementation}

Since most of the bandwidth during the communication phase between the server and clients is used for transmitting the model weights in the fully connected and convolutional layers, we focus on learning FP8 representations for these. Hence, we simulate a mixed-precision training scenario where these parameters are trained in FP8, whereas bias parameters and normalization layer parameters (GroupNorm for convolutional networks and LayerNorm for transformers) are trained in FP32, since initial experiments showed these to be sensitive to quantization. Note that this only results in a negligible impact on the client-server communication, since these parameters account for less than 2 \% of the total parameter count in the models.

An illustration of how mixed-precision training can be simulated using QAT is shown in Figure \ref{qat-diagram}. The diagram illustrates the necessary operations for quantizing most layers in the networks we use. A notable exception is the self-attention layers in the transformer model. In these layers, we perform calculations of the keys, queries and values, as well as calculation of the dot-product attention scores in FP8. Softmax normalization of the attention scores is done in FP32, similar to how it is done for other activation functions.

For all experiments, we have used $m = 3$ and $e = 4$ bits for the mantissa and exponent respectively. In addition, when performing server-side optimization of weight aggregation, we use 5 gradient descent steps when optimizing for the quantized weights in Equation \eqref{w_optimize}, where the learning rate was selected using grid search in $\{0.01, 0.1, 1\}$, and 50 grid points when optimizing the range values in Equation \eqref{alpha_optimize}.

\subsection{Results}

The results are shown in Table \ref{tab:results}, where we present the final centralized test accuracy and standard deviation across three random seeds, as well as the average communication gain compared to an FP32 baseline. In order to compare communication costs, we do not pick a common accuracy threshold for all methods, but instead calculate the gains individually for each method as the reduction in communicated bytes compared to FP32 training at the maximum accuracy reached by both FP32 and FP8. 

In Table \ref{tab:results} it can be seen that for most datasets and methods, FP8-FedAvg-UQ achieves similar test accuracy as the FP32 baseline, which results in communication gains of 4.2x on average, and for some experiments larger than 9x. Note that even though FP8 quantization sometimes results in a small accuracy drop, a large communication gain is still possible due to less data being transferred in each communication round. However, in certain cases, for example when applying FP8 quantization to LeNet on CIFAR100, accuracy increases significantly compared to the FP32 baseline. For these experiments, we observed that FP8 quantization to some extent prevented overfitting to the local client datasets, and thereby acted as a regularizer. This effect of quantization has been observed in other studies as well \cite{askarihemmat2022qreg}. 

Table \ref{tab:results} further shows results for i.i.d.\ and a more realistic non-i.i.d. settings. Here, we see an expected accuracy drop when moving from the i.i.d.\ to the non-i.i.d. setting, yet our method shows gains in both scenarios. We note that our method could potentially be combined with more advanced federated learning algorithms in heterogenous settings to improve accuracy levels in the non-i.i.d. setting as well.

Finally, we note that the server-side optimization in most scenarios yields additional performance improvements, with an average gain of 4.5x. The impact of this additional optimization is most notable when the FP8 quantization results in an accuracy drop compared to FP32. On the contrary, when there is no accuracy loss, the improvement due to server optimization is smaller. Overall, the quantized communication in combination with server-side optimization results in communication gains greater than 2.9x compared to FP32 across all tasks and models.

\begin{table*}[ht]
    \centering
    \caption{Final test accuracy and communication gain compared to FP32 FedAvg for our proposed methods on CIFAR10/CIFAR100 and Google SpeechCommands.}
    \resizebox{.8\linewidth}{!}{
    \begin{tabular}{c|c|ccc}
    \toprule
    Model & Setting & FP32 FedAvg & FP8 FedAvg-UQ & FP8 FedAvg-UQ+   \\ \midrule
    \multicolumn{5}{c}{\bf CIFAR10} \\ \midrule
    \multirow{2}{*}{LeNet} & i.i.d.\  & $82.1 \pm 0.1$ / $1\times$ & $82.0 \pm 0.1$ / $4.1\times$ & $82.2 \pm 0.3$ / $4.7\times$ \\
    & Dir(0.3) & $77.1 \pm 0.4$ / $1\times$ & $77.3 \pm 0.9$ / $3.9\times$ & $77.7 \pm 0.5$ / $3.9\times$ \\
    \multirow{2}{*}{ResNet18} & i.i.d\ & $92.0 \pm 0.1$ / $1\times$ & $91.1 \pm 0.2$ / $3.4\times$ & $92.0 \pm 0.1$ / $3.9\times$ \\
    & Dir(0.3) & $85.5 \pm 0.5$ / $1\times$ & $87.4 \pm 0.7$ / $5.2\times$ & $87.5 \pm 0.5$ / $5.2\times$ \\ \midrule
    \multicolumn{5}{c}{\bf CIFAR100} \\ \midrule
    \multirow{2}{*}{LeNet} & i.i.d.\  & $43.0 \pm 0.3$ / $ 1 \times$ & $44.8 \pm 0.4$ / $6.0 \times$ & $ 44.9 \pm 0.5$ / $6.0 \times$  \\
    & Dir(0.3) & $38.3 \pm 0.7$ / $1 \times$ & $41.1 \pm 0.3$  / $9.1 \times$ & $41.3 \pm 0.7$ / $9.5 \times$ \\
    \multirow{2}{*}{ResNet18} & i.i.d\ & $64.6 \pm 0.3$ / $ 1\times$ & $64.0 \pm 0.2$ / $3.5 \times$ & $64.6 \pm 0.1$ / $4.0 \times$  \\
    & Dir(0.3) &  $56.1 \pm 0.7$ / $1 \times$ &$55.4 \pm 0.6$ / $3.6 \times$ & $55.5 \pm 0.6$ / $3.6 \times$ \\ \midrule
    \multicolumn{5}{c}{\bf SpeechCommands} \\ \midrule
    \multirow{2}{*}{MatchboxNet} & i.i.d.\ & $91.5 \pm 0.3$ / $1 \times$ & $90.0 \pm 0.4$ / $3.5 \times$ & $90.8 \pm 0.4$ / $3.4\times$   \\
    & speaker-id & $79.6 \pm 0.7$ / $1 \times$ & $75.4 \pm 0.6$ / $3.1 \times$ & $77.0 \pm 1.3$ / $3.3 \times$  \\
    \multirow{2}{*}{KWT-1} & i.i.d\ & $91.4 \pm 0.4$ / $1 \times$ & $89.2 \pm 0.3$ / $2.3 \times$ & $90.7 \pm 0.2$ / $2.9 \times$  \\
    & speaker-id & $83.2 \pm 0.2$ / $1 \times$ & $79.6 \pm 0.5$ / $2.9 \times$ & $82.4 \pm 0.8$ / $3.7 \times$ \\ 
    \midrule
    Average gain & \multicolumn{1}{c}{} & \multicolumn{1}{r}{$1 \times$} & \multicolumn{1}{r}{$4.2 \times$} & \multicolumn{1}{r}{$4.5 \times$}  \\
    \bottomrule
    \end{tabular}
    }
    \label{tab:results}
\end{table*}

\subsection{Ablation studies}

\begin{table*}[t]
    \centering
    \caption{Final test accuracy on CIFAR10 and CIFAR100 (i.i.d.) for deterministic/stochastic QAT without quantized communication compared to FP32.}
    \resizebox{.5\linewidth}{!}{
    \begin{tabular}{c|ccc} 
    \toprule
     Model & FP32 FedAvg & det.\ QAT & rand.\ QAT  \\ \midrule
     \multicolumn{4}{c}{\bf CIFAR10} \\ \midrule
     LeNet  & $82.1 \pm 0.1$ & $82.1 \pm 0.2$  & $ 82.0 \pm 0.2$  \\
     ResNet18 & $92.0 \pm 0.1$ & $91.9 \pm 0.2$ & $91.8 \pm 0.3$ \\ \midrule
    \multicolumn{4}{c}{\bf CIFAR100} \\ \midrule
     LeNet & $43.0 \pm 0.3$ &$44.4 \pm 0.5$  & $43.7 \pm 0.6$  \\
     ResNet18 & $64.6 \pm 0.3$ & $64.5 \pm 0.1 $ & $63.5 \pm 0.5$ \\ \bottomrule
    \end{tabular}
    }
    \label{tab:qat_ablation}
    \centering
    \caption{Final test accuracy on CIFAR10 and CIFAR100 (i.i.d.) for different quantized communication (CQ) methods with deterministic QAT compared to FP32.}
    \resizebox{.5\linewidth}{!}{
    \begin{tabular}{c|ccc} \toprule
     Model & FP32 FedAvg  & det.\ CQ & rand.\ CQ \\ \midrule
    \multicolumn{4}{c}{\bf CIFAR10} \\ \midrule
     LeNet  & $82.1 \pm 0.1$ & $80.1 \pm 0.1$  & $ 82.0 \pm 0.1$  \\
     ResNet18 & $92.0 \pm 0.1$ & $91.1 \pm 0.1$ & $91.7 \pm 0.2$ \\ \midrule
     \multicolumn{4}{c}{\bf CIFAR100} \\ \midrule
     LeNet & $43.0 \pm 0.3$ & $38.0 \pm 0.4$  & $44.8 \pm 0.4$   \\
     ResNet18 & $64.6 \pm 0.3$ & $62.5 \pm 0.9$   & $64.0 \pm 0.2$\\ \bottomrule
    \end{tabular}
    }
    \label{tab:cq_ablation}
    \hfill
\end{table*}

\begin{figure*}
    \centering
    \includegraphics[width=\linewidth]{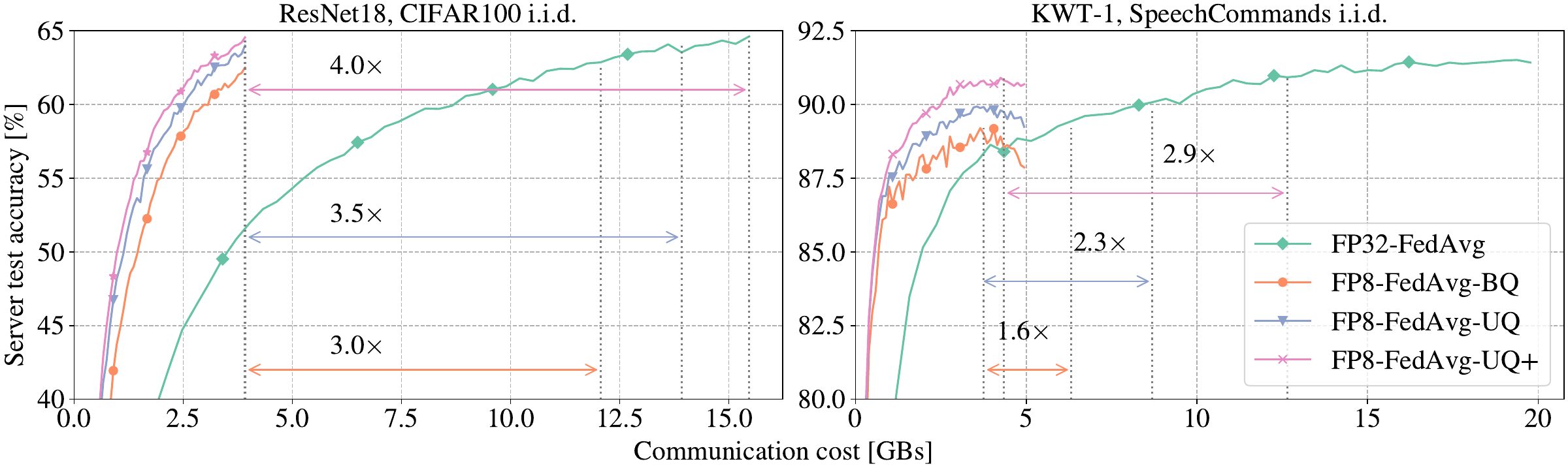}
    \caption{Server test accuracy versus communication cost for FP32 FedAvg, and FP8 QAT with biased (BQ)/unbiased (UQ) quantization for communication, and server-side optimization (UQ+). For each method, the communication gain compared to FP32 has been highlighted with arrows.}
    \label{communication}
\end{figure*}

Next, we ablate the use of deterministic and stochastic quantization in order to validate our design choices. Table \ref{tab:qat_ablation} shows the server test accuracy when using the two different quantization methods for the on-device QAT training. In general, we observe that performing local QAT does not severely impact the accuracy of the global model. In some cases, it can even increase test accuracy, which is the case for LeNet on CIFAR100, which we attribute to the quantization noise acting as a regularizer during training. Overall, deterministic quantization works slightly better, which can also be understood intuitively, since in each forward pass through the network, deterministic quantization results in a smaller quantization error. We refer to Appendix \ref{qat-section} for more details about QAT convergence.

In Table \ref{tab:cq_ablation}, we ablate the effect of deterministic and stochastic quantization in server-device communication, and it is clear that stochastic quantization results in both higher accuracy and gain. This is in agreement with Remark \ref{r.unbiased}, which shows that stochastic quantization is important for the convergence of our algorithm. In addition, we show the server test accuracy as a function of communication cost for different methods in Figure \ref{communication}. Here we can clearly see the gain arising from quantized communication, as well as the benefits of stochastic quantization and server-side optimization.

\section{Conclusions and Future Work}
In this work, we show that on-device FP8 QAT training combined with quantized communication can be integrated into a federated learning setting with a well-designed algorithm. Our results show that this can be done with minimal drop in model predictive performance, while obtaining large savings in terms of communication cost. This opens up a wide range of possibilities in terms of exploiting client heterogeneity. For example, it allows for combining devices with different computational capabilities, which could involve training with different levels of precision in different clients. The design of our algorithm is well-motivated by theory. We show results for various different models and datasets, and ablation studies to validate design choices.

Furthermore, our paper has important practical implications for machine learning hardware and system design, our results suggest that both deterministic and stochastic modes need to be supported by hardware FP8 quantizer for better training convergence in a distributed or federated setting.

Finally, since the use of low-precision number formats is orthogonal to the optimization method itself, our proposed method may be extended beyond standard federated averaging. FP8 local training may therefore be combined with more advanced optimization techniques that deal with problems such as client drift. We leave this as future work and hope our work will inspire the wider research community to explore different combinations of reduced precision floating point number formats in a federated learning setting.

\newpage

\bibliography{references}
\bibliographystyle{plain}

\onecolumn
\appendix
For FedAvg convergence proof, our analysis builds on \cite{karimireddy2020scaffold,li2017training,acar2021federated}. \cite{karimireddy2020scaffold,acar2021federated} focus on debiasing the local losses in a standard non-quantized federated learning setting. Differently, we show convergence using quantization aware training in federated learning. We can further extend our analysis to use the sophisticated debiasing methods for better heterogeneity control. \cite{li2017training} proves convergence of different quantization aware training schemes in a centralized non-federated setting. Differently, we give convergence of quantization aware training in a distributed federated learning setting.
Additionally, we give a proof for more general non-uniform quantization grids such as FP8, which is different from the uniform quantization consideration in \cite{li2017training}.

\section{Quantization Function}
\begin{definition}[Quantization]\label{d.q.0} 
For an unquantized number $x$, we define the quantization of $x$ as
\begin{align}\nonumber
    Q_{\text{rand}}(x) = s
    \begin{cases} 
      \ceil*{\frac{x}{s}} & p \leq \frac{x}{s} - \floor*{\frac{x}{s}}\\
      \floor*{\frac{x}{s}} & \text{otherwise,}
   \end{cases}
   , \ \ \ \ Q_\text{det}({ x}) = s\round*{\frac{x}{s}},
\end{align}
where $p\in[0,1]$ is a random variable. We omit the parameter of quantization for the sake of simplicity in the notation. We overload the notation and define quantization of a vector $\mathbf{x}\in\mathbb{R}^d$ as the element-wise quantization of the vector, $Q({\mathbf x})=\left[Q\left(\left[{\mathbf x}\right]_1\right),Q\left(\left[{\mathbf x}\right]_2\right),\ldots,Q\left(\left[{\mathbf x}\right]_d\right)\right]^T$
\end{definition}

Let's define the quantization error.
\begin{definition}[Quantization Error]\label{d.q.1} Let $r_Q\left(\mathbf{w}\right)=Q\left(\mathbf{w}\right)-\mathbf{w}$. 
\end{definition}

Note that if $Er_Q\left(\mathbf{w}\right)=\mathbf{0}$, we have an unbiased quantization as in our model transmission where expectation is over the randomness of the quantization.

Note that we simplified the definition by ignoring the quantization based learnable parameters such as $\alpha$ and $\beta$ in our proof. Hence, we redefine them here.

\section{Convergence Analysis of Quantization-aware Training (QAT)} \label{qat-section}
As a warmup, we provide the convergence analysis of QAT training on a single machine, similar to the one in \cite{li2017training}.

We want to find {\it a quantized} model that solves $\min_\mathbf{w\in\mathbb{R}^d} F(\mathbf{w})$. We start with an unquantized model as $\mathbf{w}_1$ and use QAT training as $\mathbf{w}_{t+1} = \mathbf{w}_t -\eta_t \nabla F\left(Q\left(\mathbf{w}_t\right);\xi_t \right)$ where $\eta_t$ is learning rate and $\xi_t$ controls randomness of SGD at iterate $t$. Let us define the best model as $\mathbf{w}_{*}=\arg\min_\mathbf{w} F(\mathbf{w})$.

Our analysis is based on the following assumptions on the objective function $F$.

\begin{assumption}[Convexity]\label{a.f.1}
We assume that $F$ is differentiable and convex, i.e.,
$$-\left\langle \nabla F ({\mathbf x}), {\mathbf x}- {\mathbf y}\right\rangle \leq -F({\mathbf x}) +F({\mathbf y}) \ \ \ \forall {\mathbf x}, {\mathbf y}.$$
\end{assumption}

\begin{assumption}[Bounded Unbiased Gradients]\label{a.f.4}
We assume the gradients are unbiased and bounded.
$$E_\xi[\nabla F\left(\mathbf{x};\xi \right)] = \nabla F\left(\mathbf{x} \right),\ \  E_\xi\|\nabla F\left(\mathbf{x};\xi \right)\|^2_2 \le G^2\ \ \ \forall {\mathbf x}.$$
where $\xi$ defines the randomness due to stochastic gradient estimator. The algorithm draws $\nabla F\left(\mathbf{x};\xi \right)$ instead of $\nabla F\left(\mathbf{x}\right)$.
\end{assumption}

\begin{assumption}[Bounded Quantization Scales]\label{a.q.1}
We assume the scales $s_i$ are uniformly upper bounded during the training by a constant $S$.
\end{assumption}

Next, we provide an upper bound on the quantization error.

\begin{lemma}\label{l.r.1}
If assumption \ref{a.q.1} 
holds, we have,
    $$E\|r_Q\left({\mathbf w}\right)\|_2\le \sqrt{d}S$$
\end{lemma}

\begin{proof}
  Each dimension of $r_Q\left({\mathbf w}\right)$ can be at max $S$. We have $d$ dimensions. Hence, $E\|r_Q\left({\mathbf w}\right)\|_2\le \sqrt{d}S$  
\end{proof}

We can then prove the following lemma on the $t$-th iteration of QAT training. 
\begin{lemma}[QAT step update]\label{l.r.2}
If assumptions \ref{a.f.1}, 
\ref{a.f.4}, and \ref{a.q.1}
hold and $\eta_t = \frac{1}{\sqrt{T}}$, we have,
$$E\left[F\left(Q\left(\mathbf{w}_t\right)\right)-F\left(\mathbf{w}_{*}\right)\right] \leq \frac{\sqrt{T}}{2}\left[-E\|\mathbf{w}_{t+1}-\mathbf{w}_{*}\|^2_2+ E\|\mathbf{w}_{t}-\mathbf{w}_{*}\|^2_2\right]+G\sqrt{d}S+ \frac{1}{2\sqrt{T}} G^2$$
\end{lemma}

\begin{proof}

Based on the update $\mathbf{w}_{t+1} = \mathbf{w}_t -\eta_t \nabla F\left(Q\left(\mathbf{w}_t\right);\xi_t \right)$ in the $t$-th iteration of QAT training, 
\begin{align}
    E&\|\mathbf{w}_{t+1}-\mathbf{w}_{*}\|_2^2=E\|\mathbf{w}_t -\eta_t \nabla F\left(Q\left(\mathbf{w}_t\right);\xi_t \right)-\mathbf{w}_{*}\|^2_2\nonumber\\
    &=E\|\mathbf{w}_{t}-\mathbf{w}_{*}\|^2_2 - 2\eta_t E\left\langle \nabla F\left(Q\left(\mathbf{w}_t\right);\xi_t \right),\mathbf{w}_t-\mathbf{w}_{*}\right\rangle + \eta_t^2 E\|\nabla F\left(Q\left(\mathbf{w}_t\right);\xi_t \right)\|^2_2 \nonumber\\
    &\le E\|\mathbf{w}_{t}-\mathbf{w}_{*}\|^2_2 - 2\eta_t E\left\langle \nabla F\left(Q\left(\mathbf{w}_t\right);\xi_t \right),\mathbf{w}_t-\mathbf{w}_{*}\right\rangle + \eta_t^2 G^2\nonumber\\
    &= E\|\mathbf{w}_{t}-\mathbf{w}_{*}\|^2_2 - 2\eta_t E\left\langle \nabla F\left(Q\left(\mathbf{w}_t\right) \right),\mathbf{w}_t-\mathbf{w}_{*}\right\rangle + \eta_t^2 G^2\nonumber\\
    &= E\|\mathbf{w}_{t}-\mathbf{w}_{*}\|^2_2 - 2\eta_t E\left\langle \nabla F\left(Q\left(\mathbf{w}_t\right) \right),Q\left(\mathbf{w}_t\right)-\mathbf{w}_{*}\right\rangle- 2\eta_t E\left\langle \nabla F\left(Q\left(\mathbf{w}_t\right) \right),\mathbf{w}_t-Q\left(\mathbf{w}_t\right)\right\rangle + \eta_t^2 G^2\nonumber\\
    &\le E\|\mathbf{w}_{t}-\mathbf{w}_{*}\|^2_2 - 2\eta_tE\left(F\left(Q\left(\mathbf{w}_t\right)\right)-F\left(\mathbf{w}_{*}\right)\right)- 2\eta_t E\left\langle \nabla F\left(Q\left(\mathbf{w}_t\right) \right),\mathbf{w}_t-Q\left(\mathbf{w}_t\right)\right\rangle + \eta_t^2 G^2\nonumber\\
    &\le E\|\mathbf{w}_{t}-\mathbf{w}_{*}\|^2_2 - 2\eta_tE\left(F\left(Q\left(\mathbf{w}_t\right)\right)-F\left(\mathbf{w}_{*}\right)\right)+ 2\eta_t E\left\| \nabla F\left(Q\left(\mathbf{w}_t\right) \right)\right\|_2\left\|r\left(\mathbf{w}_{t}\right)\right\|_2 + \eta_t^2 G^2\nonumber\\    
    &\le E\|\mathbf{w}_{t}-\mathbf{w}_{*}\|^2_2 - 2\eta_tE\left(F\left(Q\left(\mathbf{w}_t\right)\right)-F\left(\mathbf{w}_{*}\right)\right)+ 2\eta_t G E\left\| r\left(\mathbf{w}_{t}\right)\right\|_2 + \eta_t^2 G^2\nonumber\\
    &\le E\|\mathbf{w}_{t}-\mathbf{w}_{*}\|^2_2 - 2\eta_tE\left(F\left(Q\left(\mathbf{w}_t\right)\right)-F\left(\mathbf{w}_{*}\right)\right)+ 2\eta_t G \sqrt{d}S + \eta_t^2 G^2\nonumber
\end{align}
where A.\ref{a.f.4}, A.\ref{a.f.1}, $\langle\mathbf{x}, \mathbf{y}\rangle \le \|\mathbf{x}\|_2\|\mathbf{y}\|_2$, A.\ref{a.f.4}, and Lemma \ref{l.r.1} are used respectively in the inequalities. We use the fact that gradients are unbiased as well, A.\ref{a.f.4}. Let $\eta_t=\frac{1}{\sqrt{T}}$. Note that the same rate can be obtained with a decreasing learning rate scheme with a couple extra steps. Rearranging the terms and dividing with the learning rate give the Lemma.
\end{proof}

By the telescoping sum of Lemma~\ref{l.r.2} over all iterations $t=1,\dotsc,T$, we can prove the convergence of QAT training. 

\begin{theorem}[QAT Convergence]\label{thm.qat}
    \textbf{For a convex
    function 
    with bounded unbiased stochastic gradients using a quantization method with bounded scales, we have
    $$E\left[F\left(Q\left(\mathbf{w}_\tau\right)\right)-F\left(\mathbf{w}_{*}\right)\right] = O\left(\frac{1}{\sqrt{T}}\left(\|\mathbf{w}_{1}-\mathbf{w}_{*}\|^2_2+G^2\right)+G\sqrt{d}S\right)$$
    where $\tau$ is a random variable that takes values in $\{1,2,\ldots,T\}$ with equal probability\footnote{We could further derive a model bound using Jensen on LHS since the function is convex. This allows us to avoid introducing another random variable $\tau$, and would give LHS as $E\left[F\left(\frac{1}{T}\sum_{t=1}^TQ\left(\mathbf{w}_t\right)\right)-F\left(\mathbf{w}_{*}\right)\right]$. However the model, $\frac{1}{T}\sum_{t=1}^TQ\left(\mathbf{w}_t\right)$, is not necessarily a quantized model. Since we are interested in quantized model performance, we further need to argue that the quantization error of the averaged model is small and that would not change the rate. To avoid these extra steps, we introduced another random variable, $\tau$, for the sake of simplicity in the proof.}, $T$ is the number of iterations, $\mathbf{w}_1$ is the initial model and $\mathbf{w}_*$ is the optimal model $\mathbf{w}_*\in\arg\min_{\mathbf{w}} F(\mathbf{w})$, and the remaining constants are defined in the assumptions.
    }
\end{theorem}

\begin{proof}
If we average Lemma \ref{l.r.2} for all iterations we get,

\begin{align}
    E\left[\left[\frac{1}{T}\sum_{t=1}^{T}F\left(Q\left(\mathbf{w}_t\right)\right)\right]-F\left(\mathbf{w}_{*}\right)\right] &\leq \frac{1}{2\sqrt{T}}\left[-E\|\mathbf{w}_{T+1}-\mathbf{w}_{*}\|^2_2+ E\|\mathbf{w}_{1}-\mathbf{w}_{*}\|^2_2\right]+G\sqrt{d}S+ \frac{1}{2\sqrt{T}} G^2\nonumber\\
    &\leq \frac{1}{2\sqrt{T}} \|\mathbf{w}_{1}-\mathbf{w}_{*}\|^2_2+G\sqrt{d}S+ \frac{1}{2\sqrt{T}} G^2\nonumber
\end{align}

Note that 
LHS is the same if we choose $Q\left(\mathbf{w}_t\right)$ at random from all iterations with equal probability.
\end{proof}

\begin{remark}\label{r.1}
    Note that the proof uses a bound on the quantization error in the form of Lemma \ref{l.r.1}. Deterministic quantization would have a smaller bound on the norm of the quantization error, $E\left\|r_Q(w)\right\|_2$, compared to the stochastic quantization. This motivates the use of deterministic quantization during the training phase.
\end{remark}

\begin{remark}\label{r.2}
    LHS of the convergence rate in Theorem \ref{thm.qat} has two terms. First term decays with $O\left(\frac{1}{\sqrt{T}}\right)$ which is similar to the SGD rate. The second term is a constant. This constant term accounts for irreducible loss due to quantization.
\end{remark}

\section{Convergence Analysis of FP8FedAvg-UQ}\label{A.flq}

We note that $F_k$ is the local loss at device $k\in[K]$ and $F$ is the average of local functions, .i.e $F\left(\mathbf{x}\right)=\frac{1}{K}\sum_{k=1}^KF_k\left(\mathbf{w}\right)$. We assume the number of data points in each device is the same so that $F$ is a non-weighted average of local functions for the sake of simplicity. We note that results can be adjusted easily for non-equal dataset size cases. We denote $\mathbf{w}^*$ as the optimal model of the global loss, .i.e $\arg\min F\left(\mathbf{w}\right)$. For simplicity, we consider the balanced clients $n_k = \frac{n}{K}$ in our proof. However, the proof can be extended to the general imbalanced case similar to \cite{mcmahan2017communication}.

\begin{assumption}[Smoothness]\label{a.f.3}
We assume the functions are $L$ smooth.
$$\|\nabla F_k({\mathbf x}) - \nabla F_k({\mathbf y})\|_2 \le L \|{\mathbf x} - {\mathbf y}\|_2 \ \ \ \forall {\mathbf x}, {\mathbf y}, k.$$
\end{assumption}

\begin{property}\label{p.c.1}If we have smooth and convex functions, as in \cite{karimireddy2020scaffold,nesterov2018lectures,acar2021federated},
for all $\mathbf{w},\mathbf{x}, \mathbf{y}$,
$$-\left\langle \nabla F_k\left(\mathbf{w}\right), \mathbf{y} -\mathbf{x} \right\rangle\le - F_k\left(\mathbf{y}\right) + F_k\left(\mathbf{x}\right) + \frac{L}{2}\left\|\mathbf{y}-\mathbf{w}\right\|^2_2.$$
\end{property}

\subsection{Lemmas on the Stochastic Quantization for Model Communication}
\begin{lemma}\label{l.r.3.reset}
Stochastic quantization is unbiased, i.e.,
$$Er_{Q_{\text{rand}}}\left(\mathbf{x}\right)=\mathbf{0}.$$
\end{lemma}
\begin{proof}
It follows directly from the definition as,
\begin{align}
    Er_{Q_{\text{rand}}}\left(\mathbf{x}\right)&=EQ_{\text{rand}}\left(\mathbf{x}\right)-\mathbf{x}= s\left(\frac{{\mathbf x}}{s} - \left\lfloor\frac{{\mathbf x}}{s}\right\rfloor\right)\odot\left(\left\lfloor\frac{{\mathbf x}}{s}\right\rfloor+1\right) + s\left(1-\frac{{\mathbf x}}{s} + \left\lfloor\frac{{\mathbf x}}{s}\right\rfloor\right)\odot\left\lfloor\frac{{\mathbf x}}{s}\right\rfloor - \mathbf{x}=\mathbf{0}\nonumber
\end{align}
where $\odot$ denotes the element-wise product.
\end{proof}

\begin{lemma}\label{l.q.1}
Let $Q_\text{rand}$ be the stochastic unbiased quantization satisfying assumption \ref{a.q.1}
. Then we have,
$$E\left\|r_{Q_\text{rand}}\left({\mathbf x}\right)\right\|^2_2\le S\left\|{\mathbf x}\right\|_1\le S \sqrt{d} \left\|{\mathbf x}\right\|_2$$
\end{lemma}
\begin{proof}
Let's start with a scalar case and we extend it to a vector case.
\begin{align}
E\left|r_{Q_\text{rand}}\left({x}\right)\right|^2
&=s^2\left(\frac{x}{s} - \left\lfloor\frac{x}{s}\right\rfloor\right)\left(\left\lfloor\frac{x}{s}\right\rfloor+1-\frac{x}{s}\right)^2 + s^2\left(1-\frac{x}{s} + \left\lfloor\frac{x}{s}\right\rfloor\right)\left(\left\lfloor\frac{x}{s}\right\rfloor-\frac{x}{s}\right)^2\nonumber\\
&=s^2\left(\frac{x}{s} - \left\lfloor\frac{x}{s}\right\rfloor\right)\left(1+\left\lfloor\frac{x}{s}\right\rfloor-\frac{x}{s}\right)\le s^2\min\left(\frac{x}{s} - \left\lfloor\frac{x}{s}\right\rfloor, 1-\frac{x}{s} + \left\lfloor\frac{x}{s}\right\rfloor\right)\nonumber\\
&\le s^2\left|\frac{x}{s} \right| \le S |x|\nonumber
\end{align}
where inequalities follow from the fact that $\frac{x}{s} - \left\lfloor\frac{x}{s}\right\rfloor\le 1$.

We can add the scalar variances to bound a vector variance as,

\begin{align}
E\left\|r_{Q_\text{rand}}\left({\mathbf x}\right)\right\|^2_2=\sum_{i\in[d]} E\left\|r_{Q_\text{rand}}\left(\left[{\mathbf x}\right]_i\right)\right\|^2_2\le S \sum_{i\in[d]}\left|\left[{\mathbf x}\right]_i\right|=S\left\|{\mathbf x}\right\|_1\le S \sqrt{d} \left\|{\mathbf x}\right\|_2\nonumber
\end{align}
where we use Cauchy–Schwarz inequality in the last step.
\end{proof}

\begin{lemma}[Quantization Error Decomposition]\label{l.q.2}

Let assumption \ref{a.q.1} holds. Both uniform and FP8 quantization satisfies,
$$E\left|r_Q\left(Q\left({x}\right)+{y}\right)\right|^2\le S \left|{y}\right|.$$
For $d$ dimensional vectors we get,
$$E\left\|r_Q\left(Q\left(\mathbf{x}\right)+\mathbf{y}\right)\right\|^2_2\le S \sqrt{d}\left\|\mathbf{y}\right\|_2.$$
\end{lemma}

\begin{proof}
We give a proof for scalar case. Vector version comes from upper bounding scalar case using Cauchy-Schwarz. Note that variance is higher for randomized quantization so let's prove the bound for $Q_{\text{rand}}$. Due to symmetry, we can assume $Q_{\text{rand}}\left({ x}\right)\ge { 0}$.

Let's define grid points as $g_i$ where $g_0=0$ and $g_{i+1}>g_i$. Note that due to quantization definitions, we have $g_i \% (g_{i+1}-g_{i}) = 0$, .i.e $\exists k \in \mathcal{Z}^+$ such that $g_i = k \left(g_{i+1}-g_{i}\right)$. Furthermore, we have a finer resolution close to $0$, .i.e $g_{i+1}-g_{i} \ge g_{i}-g_{i-1}$.

We extensively use a step in Lemma \ref{l.q.1} as,
$$E\left|r_{Q_\text{rand}}\left(z\right)\right|^2=s^2q_z\left(1-q_z\right)\le s^2\min\left(q_z,1-q_z\right)\le s^2 \left|\frac{z}{s}\right|=s|z|$$
where $q_z=\frac{z}{s}-\left\lfloor\frac{z}{s}\right\rfloor$. We use this relation by plugging in $z=Q_\text{rand}\left({x}\right) + y$ and investigating $q_{Q_\text{rand}\left({x}\right) + y}$.

Since $Q_{\text{rand}}\left({x}\right)$ is already quantized, $\exists i\ge0$ such that $Q_{\text{rand}}\left({x}\right)=g_i$. Let $g_{j+1}> Q_\text{rand}\left({x}\right) + y \ge g_j$. 

Let $y = \delta + g_j - g_i$. Then we have $g_{j+1}> \delta + g_j \ge g_j \implies g_{j+1} - g_j> \delta \ge 0$. We also know $g_{j+1}-g_i> y \ge g_j-g_i$.

We have, by definition,

$$q_{Q_\text{rand}\left({x}\right) + y}=\frac{g_j + \delta}{g_{j+1} - g_j}-\left\lfloor\frac{g_j + \delta}{g_{j+1} - g_j}\right\rfloor=\frac{\delta}{g_{j+1} - g_j}-\left\lfloor\frac{\delta}{g_{j+1} - g_j}\right\rfloor=q_{\delta}$$

since $g_j$ is a multiple $g_{j+1} - g_j$.

Let's look at different cases.

{\noindent \it Case $i\le j$}

Note that $ g_j - g_i \ge 0$ so that $|y|=|\delta + g_j - g_i|\ge|\delta|$. 
Then, we have,
$$E\left|r_{Q_\text{rand}}\left(Q\left({x}\right) + y\right)\right|^2\le \left(g_{j+1} - g_j\right)^2\min\left(q_{ \delta},1-q_{\delta}\right)\leq \left(g_{j+1} - g_j\right)|\delta|\leq S|\delta|\leq S|y|\qed.$$

{\noindent \it Case $i>j+1$}

Note that $ g_{j+1} - g_i < 0$ and $y$ is negative. Let's look at magnitude of $y$ and $\delta$.
$$0 > g_{j+1}-g_i> y \ge g_j-g_i \implies |y|> g_i - g_{j+1} \ge g_{j+2} - g_{j+1}.$$
We already know that $g_{j+1} - g_j> \delta \ge 0$. Then we have,
$$|y|> g_{j+2} - g_{j+1} \ge g_{j+1} - g_j> \delta $$
Since $|y|>|\delta|$, we get,
$$E\left|r_{Q_\text{rand}}\left(Q\left({x}\right) + y\right)\right|^2\le \left(g_{j+1} - g_j\right)^2\min\left(q_{ \delta},1-q_{\delta}\right)\leq \left(g_{j+1} - g_j\right)|\delta|\leq S|\delta|\leq S|y|\qed.$$

{\noindent \it Case $i=j+1$}

We have $y = \delta + g_j - g_i = \delta -\left(g_{j+1} - g_j\right)$. Let's look at $q_\delta$ as,
$$q_{\delta}=\frac{\delta}{g_{j+1} - g_j}-\left\lfloor\frac{\delta}{g_{j+1} - g_j}\right\rfloor=\frac{\delta-\left(g_{j+1} - g_j\right)}{g_{j+1} - g_j}-\left\lfloor\frac{\delta-\left(g_{j+1} - g_j\right)}{g_{j+1} - g_j}\right\rfloor=\frac{y}{g_{j+1} - g_j}-\left\lfloor\frac{y}{g_{j+1} - g_j}\right\rfloor=q_y$$
Then we have,
\begin{align*}
E\left|r_{Q_\text{rand}}\left(Q\left({x}\right) + y\right)\right|^2\le &\left(g_{j+1} - g_j\right)^2\min\left(q_{ \delta},1-q_{\delta}\right)\\
=&\left(g_{j+1} - g_j\right)^2\min\left(q_{y},1-q_{y}\right)\leq \left(g_{j+1} - g_j\right)|y|\leq S|y|    
\end{align*}

\end{proof}
Please note that the above proof holds for any quantization scheme of which the grid is symmetric with respect to zero and the bin size increases monotonically going from zero to plus or minus infinity. The FP8 quantization obviously satisfies this condition.

\subsection{Lemma on a Single Communication Round} \label{single-round}

We define some useful quantities. 
For simplicity in the proof, let us define auxiliary models as,
$${\mathbf{v}}_{t,u+1}^k = {\mathbf{v}}_{t,u}^k- \eta_t \nabla F_k\left(Q_{\text{det}}\left({\mathbf{v}}_{t,u}^k\right);\xi_{t,u}^k \right)\ \  \forall u\in[U],\ \ \ \ 
    {\mathbf{v}}_{t,1}^k = Q_{\text{rand}}\left({\mathbf{w}}_{t}\right)$$

where $U$ is the total number of local updates per communication round per device. Furthermore, we can unroll the recursion as,
\begin{align*}
{\mathbf{v}}_{t,U+1}^k &= {\mathbf{v}}_{t,U}^k- \eta_t \nabla F_k\left(Q_{\text{det}}\left({\mathbf{v}}_{t,U}^k\right);\xi_{t,U}^k \right)\\
&={\mathbf{v}}_{t,UE-1}^{k}- \eta_t \nabla F_k\left(Q_{\text{det}}\left({\mathbf{v}}_{t,U-1}^k\right);\xi_{t,U-1}^k \right)- \eta_t \nabla F_k\left(Q_{\text{det}}\left({\mathbf{v}}_{t,U}^k\right);\xi_{t,U}^k \right)=\ldots\\
&=Q_{\text{rand}}\left({\mathbf{w}}_{t}\right)-\eta_t\sum_{u\in[U]}\nabla F_k\left(Q_{\text{det}}\left({\mathbf{v}}_{t,u}^k\right);\xi_{t,u}^k \right)    
\end{align*}

It is clear to see that ${\mathbf{w}}_{t+1}^k={\mathbf{v}}_{t,U+1}^k$ for active devices. Let's define inactive device ${\mathbf{w}}_{t+1}^k={\mathbf{v}}_{t,U+1}^k$ as well. Note that this is just for notation and the algorithm is unchanged. Because if $k$ is not active we do not use ${\mathbf{w}}_{t+1}^k$ in our algorithm. Let us define a drift quantity similar to \cite{karimireddy2020scaffold}.
\begin{align}
    V_t=\frac{1}{KU}\sum_{k\in[K]}\sum_{u\in[U]}E\left\|Q_{\text{rand}}\left({\mathbf{w}}_{t}\right)-Q_{\text{det}}\left({\mathbf{v}}_{t,u}^k\right)\right\|^2_2.
\end{align}
Note that if local models diverge, we get a higher $V_t$. We can obtain the following lemma for a single communication round of the FP8FedAvg-UQ algorithm. 

\begin{lemma}\label{l.fl.1}
If assumptions \ref{a.f.1}, \ref{a.f.4}, \ref{a.q.1},  
\ref{a.f.3} hold and we use an unbiased quantization for model transmission, we have,
\begin{align}
  E\|\mathbf{w}_{t+1}-\mathbf{w}_{*}\|^2_2\le& E\left\|{\mathbf{w}}_{t}-\mathbf{w}_{*} \right\|^2_2-2U\eta_tE\left(F\left(Q_{\text{rand}}\left({\mathbf{w}}_{t}\right)\right)-F\left(\mathbf{w}_{*}\right)\right)\nonumber\\
  &+\eta_tL U V_t+2S\sqrt{d} G U \eta_t+\eta_t^2U^2G^2\label{A.1}\\
  V_t \le& 18U^3S\sqrt{d} G \eta_t  + 9U^2\eta_t^2G^2\label{A.2}
\end{align}
\end{lemma}

\begin{proof}
First, we prove Eq. \ref{A.1}. 
Due to the model-to-server communication and the model aggregation on the server in the $t$-th round, we have
\begin{align}
    E\left\|\mathbf{w}_{t+1}-\mathbf{w}_{*}\right\|^2_2=&E\left\|\frac{1}{P}\sum_{k\in\mathcal{P}_t}Q_{\text{rand}}\left({\mathbf{w}}_{t+1}^k\right)-\mathbf{w}_{*}\right\|^2_2\le \frac{1}{P}E \sum_{k\in\mathcal{P}_t} \left\|Q_{\text{rand}}\left({\mathbf{w}}_{t+1}^k\right)-\mathbf{w}_{*}\right\|^2_2 \nonumber\\
    =& \frac{1}{K}\sum_{k\in[K]} E\left\|Q_{\text{rand}}\left({\mathbf{w}}_{t+1}^k\right)-\mathbf{w}_{*}\right\|^2_2\nonumber
\end{align}
where we use definition of $\mathbf{w}_{t+1}$ and triangular inequality ($\left\|\sum_{n\in[N]}a_n\right\|^2\le N \sum_{n\in[N]}\left\|a_n\right\|^2$). Lastly, we use the fact that active devices are sampled uniformly at random so that each device has an activation probability of $\frac{P}{K}$. Let's continue as
\begin{align}
    E\left\|\mathbf{w}_{t+1}-\mathbf{w}_{*}\right\|^2_2\le&\frac{1}{K} \sum_{k\in[K]} E\left\|Q_{\text{rand}}\left({\mathbf{w}}_{t+1}^k\right)-\mathbf{w}_{*}\right\|^2_2=\frac{1}{K} \sum_{k\in[K]} E\left\|r_{Q_\text{rand}}\left({\mathbf{w}}_{t+1}^k\right)+{\mathbf{w}}_{t+1}^k-\mathbf{w}_{*}\right\|^2_2\nonumber\\
    =&\frac{1}{K} \left(\sum_{k\in[K]} E\left\|r_{Q_\text{rand}}\left({\mathbf{w}}_{t+1}^k\right)\right\|^2_2+2E\left\langle r_{Q_\text{rand}}\left({\mathbf{w}}_{t+1}^k\right),{\mathbf{w}}_{t+1}^k-\mathbf{w}_{*}\right\rangle+E\left\|{\mathbf{w}}_{t+1}^k-\mathbf{w}_{*}\right\|^2_2\right)\nonumber\\
    =&\frac{1}{K} \left(\sum_{k\in[K]} E\left\|r_{Q_\text{rand}}\left({\mathbf{w}}_{t+1}^k\right)\right\|^2_2+E\left\|{\mathbf{w}}_{t+1}^k-\mathbf{w}_{*}\right\|^2_2\right)\nonumber\\
    \nonumber
\end{align}
where we use the fact that $Q_\text{rand}$ is an unbiased quantizer. Let's bound $E\left\|{\mathbf{w}}_{t+1}^k-\mathbf{w}_{*}\right\|^2$ as
\begin{align*}
E&\left\|{\mathbf{w}}_{t+1}^k-\mathbf{w}_{*}\right\|^2_2 =  E\left\|Q_{\text{rand}}\left({\mathbf{w}}_{t}\right)-\mathbf{w}_{*}-\eta_t\sum_{u\in[U]}\nabla F_k\left(Q_{\text{det}}\left({\mathbf{v}}_{t,u}^k\right);\xi_{t,u}^k \right) \right\|^2_2\nonumber\\
=& E\left\|Q_{\text{rand}}\left({\mathbf{w}}_{t}\right)-\mathbf{w}_{*} \right\|^2_2-2\eta_t\sum_{u\in[U]}E\left\langle Q_{\text{rand}}\left({\mathbf{w}}_{t}\right)-\mathbf{w}_{*},\nabla F_k\left(Q_{\text{det}}\left({\mathbf{v}}_{t,u}^k\right);\xi_{t,u}^k \right) \right\rangle\\
&+\eta_t^2E\left\|\sum_{u\in[U]}\nabla F_k\left(Q_{\text{det}}\left({\mathbf{v}}_{t,u}^k\right);\xi_{t,u}^k \right) \right\|^2_2\\
\le& E\left\|Q_{\text{rand}}\left({\mathbf{w}}_{t}\right)-\mathbf{w}_{*} \right\|^2_2-2\eta_t\sum_{u\in[U]}E\left\langle Q_{\text{rand}}\left({\mathbf{w}}_{t}\right)-\mathbf{w}_{*},\nabla F_k\left(Q_{\text{det}}\left({\mathbf{v}}_{t,u}^k\right);\xi_{t,u}^k \right) \right\rangle+\eta_t^2U^2G^2\\
=& E\left\|Q_{\text{rand}}\left({\mathbf{w}}_{t}\right)-\mathbf{w}_{*} \right\|^2_2-2\eta_t\sum_{u\in[U]}E\left\langle Q_{\text{rand}}\left({\mathbf{w}}_{t}\right)-\mathbf{w}_{*},\nabla F_k\left(Q_{\text{det}}\left({\mathbf{v}}_{t,u}^k\right) \right) \right\rangle+\eta_t^2U^2G^2\\
\le& E\left\|Q_{\text{rand}}\left({\mathbf{w}}_{t}\right)-\mathbf{w}_{*} \right\|^2_2+2\eta_t\sum_{u\in[U]}E\left[-F_k\left(Q_{\text{rand}}\left({\mathbf{w}}_{t}\right)\right)+F_k\left(\mathbf{w}_{*}\right)\right]\\
&+\eta_tL\sum_{u\in[U]}\left\|Q_{\text{rand}}\left({\mathbf{w}}_{t}\right)-Q_{\text{det}}\left({\mathbf{v}}_{t,u}^k\right)\right\|^2+\eta_t^2U^2G^2\\
=& E\left\|Q_{\text{rand}}\left({\mathbf{w}}_{t}\right)-\mathbf{w}_{*} \right\|^2_2-2U\eta_tE\left(F_k\left(Q_{\text{rand}}\left({\mathbf{w}}_{t}\right)\right)-F_k\left(\mathbf{w}_{*}\right)\right)\\
&+\eta_tL\sum_{u\in[U]}\left\|Q_{\text{rand}}\left({\mathbf{w}}_{t}\right)-Q_{\text{det}}\left({\mathbf{v}}_{t,u}^k\right)\right\|^2_2+\eta_t^2U^2G^2
\end{align*}
where we use the fact that gradients are bounded, $\nabla F_k\left(Q_{\text{det}}\left({\mathbf{v}}_{t,u}^k\right);\xi_{t,u}^k \right)$ is an unbiased gradient estimate and property \ref{p.c.1}. We further restate $E\left\|Q_{\text{rand}}\left({\mathbf{w}}_{t}\right)-\mathbf{w}_{*} \right\|^2_2$ as,
\begin{align*}
    E\left\|Q_{\text{rand}}\left({\mathbf{w}}_{t}\right)-\mathbf{w}_{*} \right\|^2_2 &= E\left\|r_{Q_{\text{rand}}}\left({\mathbf{w}}_{t}\right)+{\mathbf{w}}_{t}-\mathbf{w}_{*} \right\|^2_2\\
    &= E\left\|r_{Q_{\text{rand}}}\left({\mathbf{w}}_{t}\right) \right\|^2_2+2E\left\langle r_{Q_{\text{rand}}}\left({\mathbf{w}}_{t}\right), {\mathbf{w}}_{t}-\mathbf{w}_{*} \right\rangle+E\left\|{\mathbf{w}}_{t}-\mathbf{w}_{*} \right\|^2_2\\
    &= E\left\|r_{Q_{\text{rand}}}\left({\mathbf{w}}_{t}\right) \right\|^2_2+E\left\|{\mathbf{w}}_{t}-\mathbf{w}_{*} \right\|^2_2
\end{align*}
where we use the fact that $Q_\text{rand}$ is an unbiased quantizer. Then, we have,
\begin{align*}
E\left\|{\mathbf{w}}_{t+1}^k-\mathbf{w}_{*}\right\|^2_2 \le& E\left\|Q_{\text{rand}}\left({\mathbf{w}}_{t}\right)-\mathbf{w}_{*} \right\|^2_2-2U\eta_tE\left(F_k\left(Q_{\text{rand}}\left({\mathbf{w}}_{t}\right)\right)-F_k\left(\mathbf{w}_{*}\right)\right)\\
&+\eta_tL\sum_{u\in[U]}\left\|Q_{\text{rand}}\left({\mathbf{w}}_{t}\right)-Q_{\text{det}}\left({\mathbf{v}}_{t,u}^k\right)\right\|^2_2+\eta_t^2U^2G^2\\
=& E\left\|r_{Q_{\text{rand}}}\left({\mathbf{w}}_{t}\right) \right\|^2_2+E\left\|{\mathbf{w}}_{t}-\mathbf{w}_{*} \right\|^2_2-2U\eta_tE\left(F_k\left(Q_{\text{rand}}\left({\mathbf{w}}_{t}\right)\right)-F_k\left(\mathbf{w}_{*}\right)\right)\\
&+\eta_tL\sum_{u\in[U]}\left\|Q_{\text{rand}}\left({\mathbf{w}}_{t}\right)-Q_{\text{det}}\left({\mathbf{v}}_{t,u}^k\right)\right\|^2_2+\eta_t^2U^2G^2
\end{align*}

Using Lemma \ref{l.q.2} we have,
\begin{align*}
E\left\|r_{Q_\text{rand}}\left({\mathbf{w}}_{t+1}^k\right)\right\|^2_2&=E\left\|r_{Q_\text{rand}}\left(Q_{\text{rand}}\left({\mathbf{w}}_{t}\right)-\eta_t\sum_{u\in[U]}\nabla F_k\left(Q_{\text{det}}\left({\mathbf{v}}_{t,u}^k\right);\xi_{t,u}^k \right) \right)\right\|^2_2\\
&\le S\sqrt{d} E\left\|\eta_t\sum_{u\in[U]}\nabla F_k\left(Q_{\text{det}}\left({\mathbf{v}}_{t,u}^k\right);\xi_{t,u}^k \right)\right\|_2\le S\sqrt{d} G U \eta_t 
\end{align*}
where $U$ is the number of local iterates. Finally, we can upper bound RHS as,
\begin{align*}
E&\left\|\mathbf{w}_{t+1}-\mathbf{w}_{*}\right\|^2_2\le\frac{1}{K} \left(\sum_{k\in[K]} E\left\|r_{Q_\text{rand}}\left({\mathbf{w}}_{t+1}^k\right)\right\|^2_2+E\left\|{\mathbf{w}}_{t+1}^k-\mathbf{w}_{*}\right\|^2_2\right)\\
\le&E\left\|{\mathbf{w}}_{t}-\mathbf{w}_{*} \right\|^2_2-2U\eta_tE\left(F\left(Q_{\text{rand}}\left({\mathbf{w}}_{t}\right)\right)-F\left(\mathbf{w}_{*}\right)\right)+\frac{\eta_tL}{K}\sum_{k\in[K]}\sum_{u\in[U]}\left\|Q_{\text{rand}}\left({\mathbf{w}}_{t}\right)-Q_{\text{det}}\left({\mathbf{v}}_{t,u}^k\right)\right\|^2_2\\
&+2S\sqrt{d} G U \eta_t+\eta_t^2U^2G^2\\
=&E\left\|{\mathbf{w}}_{t}-\mathbf{w}_{*} \right\|^2_2-2U\eta_tE\left(F\left(Q_{\text{rand}}\left({\mathbf{w}}_{t}\right)\right)-F\left(\mathbf{w}_{*}\right)\right)+\eta_tL U V_t+2S\sqrt{d} G U \eta_t+\eta_t^2U^2G^2
\end{align*}
This completes Eq. \ref{A.1}'s proof. 
\begin{remark}\label{r.q.3}
    Note that we extensively use unbiasedness of stochastic quantization, $E\left\langle \text{Vector},r_{Q_\text{rand}}\left({\mathbf w}\right)\right\rangle = {\mathbf 0}$. Otherwise, we need to upper bound this term. There exists cases where a biased resetting diverges \cite{beznosikov2023biased}. Hence, stochastic quantization is needed for convergence. 
\end{remark}

Next, we prove Eq. \ref{A.2} for upper bounding the drift $V_t$ in round $t$ defined in \eqref{A.2}.
\begin{align*}
    E&\left\|Q_{\text{det}}\left({\mathbf{v}}_{t,u+1}^k\right)-Q_{\text{rand}}\left({\mathbf{w}}_{t}\right)\right\|^2_2 = E\left\|r_{Q_{\text{det}}}\left({\mathbf{v}}_{t,u+1}^k\right)+{\mathbf{v}}_{t,u+1}^k-Q_{\text{rand}}\left({\mathbf{w}}_{t}\right)\right\|^2_2\\
    =& E\left\|r_{Q_{\text{det}}}\left({\mathbf{v}}_{t,u+1}^k\right)+{\mathbf{v}}_{t,u}^k- \eta_t \nabla F_k\left(Q_{\text{det}}\left({\mathbf{v}}_{t,u}^k\right);\xi_{t,u}^k \right)-Q_{\text{rand}}\left({\mathbf{w}}_{t}\right)\right\|^2_2\\
    =& E\left\|r_{Q_{\text{det}}}\left({\mathbf{v}}_{t,u+1}^k\right)-r_{Q_{\text{det}}}\left({\mathbf{v}}_{t,u}^k\right)- \eta_t \nabla F_k\left(Q_{\text{det}}\left({\mathbf{v}}_{t,u}^k\right);\xi_{t,u}^k \right)+Q_{\text{det}}\left({\mathbf{v}}_{t,u}^k\right)-Q_{\text{rand}}\left({\mathbf{w}}_{t}\right)\right\|^2_2\\
    \le& \frac{U}{U-1}E\left\|Q_{\text{det}}\left({\mathbf{v}}_{t,u}^k\right)-Q_{\text{rand}}\left({\mathbf{w}}_{t}\right)\right\|^2_2+UE\left\|r_{Q_{\text{det}}}\left({\mathbf{v}}_{t,u+1}^k\right)-r_{Q_{\text{det}}}\left({\mathbf{v}}_{t,u}^k\right)- \eta_t \nabla F_k\left(Q_{\text{det}}\left({\mathbf{v}}_{t,u}^k\right);\xi_{t,u}^k \right)\right\|^2_2\\
    \le& \frac{U}{U-1}E\left\|Q_{\text{det}}\left({\mathbf{v}}_{t,u}^k\right)-Q_{\text{rand}}\left({\mathbf{w}}_{t}\right)\right\|^2_2\\
    &+3UE\left\|r_{Q_{\text{det}}}\left({\mathbf{v}}_{t,u+1}^k\right)\right\|^2_2+3UE\left\|r_{Q_{\text{det}}}\left({\mathbf{v}}_{t,u}^k\right)\right\|^2_2+3U\eta_t^2E\left\|\nabla F_k\left(Q_{\text{det}}\left({\mathbf{v}}_{t,u}^k\right);\xi_{t,u}^k \right)\right\|^2_2\\
    \le& \frac{U}{U-1}E\left\|Q_{\text{det}}\left({\mathbf{v}}_{t,u}^k\right)-Q_{\text{rand}}\left({\mathbf{w}}_{t}\right)\right\|_2^2+3UE\left\|r_{Q_{\text{det}}}\left({\mathbf{v}}_{t,u+1}^k\right)\right\|^2_2+3UE\left\|r_{Q_{\text{det}}}\left({\mathbf{v}}_{t,u}^k\right)\right\|^2_2 + 3U\eta_t^2G^2
\end{align*}
where we use $\left\|{\mathbf{x}}+{\mathbf{y}}\right\|^2_2\le \left(1+\frac{1}{A}\right)\left\|{\mathbf{x}}\right\|^2_2 + (A+1)\left\|{\mathbf{y}}\right\|^2_2$, triangular inequality and bound on the gradients.

Let's bound $E\left\|r_{Q_{\text{det}}}\left({\mathbf{v}}_{t,u+1}^k\right)\right\|^2_2$ using Lemma \ref{l.q.2} as,
\begin{align*}
E\left\|r_{Q_{\text{det}}}\left({\mathbf{v}}_{t,u+1}^k\right)\right\|^2_2&=E\left\|r_{Q_{\text{det}}}\left(Q_{\text{rand}}\left({\mathbf{w}}_{t}\right)-\eta_t\sum_{s\in[u]}\nabla F_k\left(Q_{\text{det}}\left({\mathbf{v}}_{t,s}^k\right);\xi_{t,s}^k\right)\right)\right\|^2_2 \\
&\le S\sqrt{d} E\left\|\eta_t\sum_{s\in[u]}\nabla F_k\left(Q_{\text{det}}\left({\mathbf{v}}_{t,s}^k\right);\xi_{t,s}^k\right)\right\|_2\le S\sqrt{d} G u \eta_t 
\end{align*}
This leads to 
\begin{align*}
    E\left\|Q_{\text{det}}\left({\mathbf{v}}_{t,u+1}^k\right)-Q_{\text{rand}}\left({\mathbf{w}}_{t}\right)\right\|^2_2 
    \le& \frac{U}{U-1}E\left\|Q_{\text{det}}\left({\mathbf{v}}_{t,u}^k\right)-Q_{\text{rand}}\left({\mathbf{w}}_{t}\right)\right\|^2_2+3UE\left\|r_{Q_{\text{det}}}\left({\mathbf{v}}_{t,u+1}^k\right)\right\|^2_2\\
    &+3UE\left\|r_{Q_{\text{det}}}\left({\mathbf{v}}_{t,u}^k\right)\right\|^2_2 + 3U\eta_t^2G^2\\
    \le& \frac{U}{U-1}E\left\|Q_{\text{det}}\left({\mathbf{v}}_{t,u}^k\right)-Q_{\text{rand}}\left({\mathbf{w}}_{t}\right)\right\|^2_2+6U^2S\sqrt{d} G\eta_t  + 3U\eta_t^2G^2
\end{align*}

Let's unroll the recursion noting that $Q_{\text{det}}\left({\mathbf{v}}_{t,1}^k\right)=Q_{\text{rand}}\left({\mathbf{w}}_{t}\right)$,
\begin{align*}
    E&\left\|Q_{\text{det}}\left({\mathbf{v}}_{t,u+1}^k\right)-Q_{\text{rand}}\left({\mathbf{w}}_{t}\right)\right\|^2_2 
    \le \frac{U}{U-1}E\left\|Q_{\text{det}}\left({\mathbf{v}}_{t,u}^k\right)-Q_{\text{rand}}\left({\mathbf{w}}_{t}\right)\right\|^2_2+6U^2S\sqrt{d} G \eta_t  + 3U\eta_t^2G^2\\
    \le&\left(\frac{U}{U-1}\right)^2E\left\|Q_{\text{det}}\left({\mathbf{v}}_{t,u-1}^k\right)-Q_{\text{rand}}\left({\mathbf{w}}_{t}\right)\right\|^2_2+\left(6U^2S\sqrt{d} G \eta_t  + 3U\eta_t^2G^2\right)\left(1+\frac{U}{U-1}\right)\\
    &\ldots\\
    \le&\left(6U^2S\sqrt{d} G \eta_t  + 3U\eta_t^2G^2\right)\left(1+\frac{U}{U-1}+\ldots+\left(\frac{U}{U-1}\right)^{u-1}\right)
\end{align*}

Let's bound the second term in the RHS as,

\begin{align*}
    1+\frac{U}{U-1}+\ldots+\left(\frac{U}{U-1}\right)^{u-1}&\le u \left(\frac{U}{U-1}\right)^{u-1} = u \left(1+\frac{1}{U-1}\right)^{u-1}\le U \left(1+\frac{1}{U-1}\right)^{U-1} \le Ue \le 3U
\end{align*}

Hence we get 
\begin{align}
\label{eq..0}E\left\|Q_{\text{det}}\left({\mathbf{v}}_{t,u+1}^k\right)-Q_{\text{rand}}\left({\mathbf{w}}_{t}\right)\right\|^2_2 \le 18U^3S\sqrt{d} G \eta_t  + 9U^2\eta_t^2G^2
\end{align}

Note that we inherently assume $U>1$ in order to have a coefficient as $\frac{U}{U-1}$. Assume $U=1$. Then we have, $V_t=0$ by definition and Eq. \ref{eq..0} holds. If we average Eq. \ref{eq..0} over $U$ and $K$ we get Eq. \ref{A.2} as,
$V_t \le 18U^3S\sqrt{d} G \eta_t  + 9U^2\eta_t^2G^2$.

\end{proof}

\subsection{Proof of the Main Theorem}

Now, we are ready to present the main theorem on the convergence of the proposed FP8FedAvg-UQ algorithm. 

\begin{theorem}[FP8FedAvg-UQ Convergence]\label{thm.fl.qat}
    \textbf{
    For convex and smooth federated losses 
    with bounded unbiased stochastic gradients using a quantization method with bounded scales during training and an unbiased quantization with bounded scales for model transfer, we have,
    $$E\left[F\left(Q\left(\mathbf{w}_\tau\right)\right)-F\left(\mathbf{w}_{*}\right)\right] = O\left(\frac{1}{\sqrt{TU}} \|\mathbf{w}_{1}-\mathbf{w}_{*}\|^2_2+\frac{1}{T}UG^2L+\frac{1}{\sqrt{T}}G\sqrt{U}\left(G+U^2S\sqrt{d}L\right) +S\sqrt{d}G\right)$$
    where $\tau$ is a random variable that takes values in $\{1,2,\ldots,T\}$ with equal probability, $T$ is the number of rounds, 
    $U$ is the total number of updates done in each round, quantization scales $s_i$ are uniformly bounded by $S$, $\mathbf{w}_1$ is the initial model, and $\mathbf{w}_*$ is an optimal solution of \eqref{standard}.
    }
\end{theorem}

Combining Eq. \ref{A.1} and $\eta_t L U$ times Eq. \ref{A.2} gives,
\begin{align*}
E\|\mathbf{w}_{t+1}-\mathbf{w}_{*}\|^2_2\le& E\left\|{\mathbf{w}}_{t}-\mathbf{w}_{*} \right\|^2_2-2U\eta_tE\left(F\left(Q_{\text{rand}}\left({\mathbf{w}}_{t}\right)\right)-F\left(\mathbf{w}_{*}\right)\right)+2S\sqrt{d} G U \eta_t+\eta_t^2U^2G^2\\&+18U^4S\sqrt{d} G \eta_t^2 L + 9U^3\eta_t^3G^2L
\end{align*}

Rearranging the terms and dividing both sides with $2U\eta_t$ gives,
\begin{align*}
    E\left[F\left(Q_{\text{rand}}\left({\mathbf{w}}_{t}\right)\right)-F\left(\mathbf{w}_{*}\right)\right]\le& \frac{1}{2U\eta_t}\left(-E\|\mathbf{w}_{t+1}-\mathbf{w}_{*}\|^2_2+ E\left\|{\mathbf{w}}_{t}-\mathbf{w}_{*} \right\|^2_2\right)\\&+S\sqrt{d}G+\frac{1}{2}\eta_tUG^2+9U^3S\sqrt{d} G \eta_t L + \frac{9}{2}U^2\eta_t^2G^2L
\end{align*}

Let $\eta_t=\frac{1}{\sqrt{UT}}$. Note that we can get the same rate with a decreasing learning rate as well. Let's average the inequality over $t$ as,

\begin{align*}
    E&\left[\left[\frac{1}{T}\sum_{t=1}^{T}F\left(Q_{\text{rand}}\left(\mathbf{w}_t\right)\right)\right]-F\left(\mathbf{w}_{*}\right)\right] \\
    &\leq \frac{1}{2\sqrt{TU}}\left[-E\|\mathbf{w}_{T+1}-\mathbf{w}_{*}\|^2_2+ E\|\mathbf{w}_{1}-\mathbf{w}_{*}\|^2_2\right]+\frac{1}{T}\frac{9}{2}UG^2L+\frac{1}{\sqrt{T}}G\sqrt{U}\left(\frac{1}{2}G+9U^2S\sqrt{d}L\right) +S\sqrt{d}G\\
    &\leq \frac{1}{2\sqrt{TU}} \|\mathbf{w}_{1}-\mathbf{w}_{*}\|^2_2+\frac{1}{T}\frac{9}{2}UG^2L+\frac{1}{\sqrt{T}}G\sqrt{U}\left(\frac{1}{2}G+9U^2S\sqrt{d}L\right) +S\sqrt{d}G\\
    &=O\left(\frac{1}{\sqrt{TU}} \|\mathbf{w}_{1}-\mathbf{w}_{*}\|^2_2+\frac{1}{T}UG^2L+\frac{1}{\sqrt{T}}G\sqrt{U}\left(G+U^2S\sqrt{d}L\right) +S\sqrt{d}G\right)\qed
\end{align*}

\end{document}